\documentclass[a4paper]{article}
\usepackage{amsmath,amsthm,amsfonts,thmtools}
\usepackage{amssymb}
\usepackage[numbers, compress]{natbib}
\usepackage[cmintegrals]{newtxmath}
\usepackage{graphicx}
\usepackage{listings}
\usepackage[font=footnotesize]{caption}
\usepackage[font=scriptsize]{subcaption}
\usepackage{euscript}
\usepackage{units}
\usepackage{enumerate}
\usepackage{xcolor}
\usepackage{todonotes}
\usepackage{booktabs}
\usepackage{tikz}
\usepackage{comment}
 
\DeclareMathAlphabet{\pazocal}{OMS}{zplm}{m}{n}

\usepackage[bookmarks=true,hidelinks]{hyperref}
\usepackage{bbm}
\usepackage{multirow}

\usepackage{mathtools}
\usepackage{subcaption} 
\usepackage{graphicx}
\usepackage{pgfplots}
\usepackage{float}
\usepackage{algpseudocode}
\usepackage{bbold}
\usepackage{microtype}
\usepackage{rotating}
\usepackage{enumitem}
\usepackage{bm}
\usepackage{hyperref}

\numberwithin{equation}{section}
\newtheorem{theorem}{Theorem}[section]

\newtheorem{lemma}[theorem]{Lemma}

\newtheorem{algorithm}[theorem]{Algorithm}
\newtheorem{definition}[theorem]{Definition}

\numberwithin{equation}{section}

\theoremstyle{definition}

\newtheoremstyle{myremarkstyle}{}{}{}{}{\bfseries}{.}{ }{}
\theoremstyle{myremarkstyle}
\declaretheorem[name=Remark,qed=$\blacksquare$,numberlike=theorem]{remark}

\makeatletter
\newcommand*{\intavg}{%
  \mint@l{-}{}%
}
\newcommand*{\mint@l}[2]{%
  \@ifnextchar\limits{%
    \mint@l{#1}%
  }{%
    \@ifnextchar\nolimits{%
      \mint@l{#1}%
    }{%
      \@ifnextchar\displaylimits{%
        \mint@l{#1}%
      }{%
        \mint@s{#2}{#1}%
      }%
    }%
  }%
}
\newcommand*{\mint@s}[2]{%
  \@ifnextchar_{%
    \mint@sub{#1}{#2}%
  }{%
    \@ifnextchar^{%
      \mint@sup{#1}{#2}%
    }{%
      \mint@{#1}{#2}{}{}%
    }%
  }%
}
\def\mint@sub#1#2_#3{%
  \@ifnextchar^{%
    \mint@sub@sup{#1}{#2}{#3}%
  }{%
    \mint@{#1}{#2}{#3}{}%
  }%
}
\def\mint@sup#1#2^#3{%
  \@ifnextchar_{%
    \mint@sub@sup{#1}{#2}{#3}%
  }{%
    \mint@{#1}{#2}{}{#3}%
  }%
}
\def\mint@sub@sup#1#2#3^#4{%
  \mint@{#1}{#2}{#3}{#4}%
}
\def\mint@sup@sub#1#2#3_#4{%
  \mint@{#1}{#2}{#4}{#3}%
}
\newcommand*{\mint@}[4]{%
  \mathop{}%
  \mkern-\thinmuskip
  \mathchoice{%
    \mint@@{#1}{#2}{#3}{#4}%
        \displaystyle\textstyle\scriptstyle
  }{%
    \mint@@{#1}{#2}{#3}{#4}%
        \textstyle\scriptstyle\scriptstyle
  }{%
    \mint@@{#1}{#2}{#3}{#4}%
        \scriptstyle\scriptscriptstyle\scriptscriptstyle
  }{%
    \mint@@{#1}{#2}{#3}{#4}%
        \scriptscriptstyle\scriptscriptstyle\scriptscriptstyle
  }%
  \mkern-\thinmuskip
  \int#1%
  \ifx\\#3\\\else_{#3}\fi
  \ifx\\#4\\\else^{#4}\fi  
}
\newcommand*{\mint@@}[7]{%
  \begingroup
    \sbox0{$#5\int\m@th$}%
    \sbox2{$#5\int_{}\m@th$}%
    \dimen2=\wd0 %
    \let\mint@limits=#1\relax
    \ifx\mint@limits\relax
      \sbox4{$#5\int_{\kern1sp}^{\kern1sp}\m@th$}%
      \ifdim\wd4>\wd2 %
        \let\mint@limits=\nolimits
      \else
        \let\mint@limits=\limits
      \fi
    \fi
    \ifx\mint@limits\displaylimits
      \ifx#5\displaystyle
        \let\mint@limits=\limits
      \fi
    \fi
    \ifx\mint@limits\limits
      \sbox0{$#7#3\m@th$}%
      \sbox2{$#7#4\m@th$}%
      \ifdim\wd0>\dimen2 %
        \dimen2=\wd0 %
      \fi
      \ifdim\wd2>\dimen2 %
        \dimen2=\wd2 %
      \fi
    \fi
    \rlap{%
      $#5%
        \vcenter{%
          \hbox to\dimen2{%
            \hss
            $#6{#2}\m@th$%
            \hss
          }%
        }%
      $%
    }%
  \endgroup
}

\def\XXint#1#2#3{{\setbox0=\hbox{$#1{#2#3}{\int}$ }
		\vcenter{\hbox{$#2#3$ }}\kern-.6\wd0}}

\renewcommand{\geq}{\geqslant}
\renewcommand{\leq}{\leqslant}

\renewcommand{\epsilon}{\varepsilon}
\renewcommand{\phi}{\varphi}

\newcommand{\R}{\mathbb{R}}
\newcommand{\N}{\mathbb{N}}

\newcommand{\U}{{\bf U}}		

\newcommand{\Prob}{\EuScript{P}}

\newcommand{\map}{\EuScript{L}}
\newcommand{\train}{\EuScript{S}}

\newcommand{\reg}{\EuScript{R}}
\newcommand{\er}{\EuScript{E}}

\begin{document}

\date{\today}

\title{Enhancing accuracy of deep learning algorithms \\by training with low-discrepancy sequences}

\author{Siddhartha Mishra \thanks{Seminar for Applied Mathematics (SAM), D-Math \newline
  ETH Z\"urich, R\"amistrasse 101, 
  Z\"urich-8092, Switzerland} and
  T. Konstantin Rusch \thanks{Seminar for Applied Mathematics (SAM), D-Math \newline
  ETH Z\"urich, R\"amistrasse 101, 
  Z\"urich-8092, Switzerland.}}

\date{\today}

\maketitle
\begin{abstract}
We propose a deep supervised learning algorithm based on low-discrepancy sequences as the training set. By a combination of theoretical arguments and extensive numerical experiments we demonstrate that the proposed algorithm significantly outperforms standard deep learning algorithms that are based on randomly chosen training data, for problems in moderately high dimensions. The proposed algorithm provides an efficient method for building inexpensive surrogates for many underlying maps in the context of scientific computing.    
\end{abstract}
\section{Introduction}
A fundamental objective of computer simulation is the \emph{prediction} of the response of a physical or biological system to inputs. This is formalized in terms of \emph{evaluation} of an underlying function (map, observable) for different inputs, i.e. 
$$
{\rm Compute}~\map(y),\quad {\rm for~}y\in Y.
$$
Here, the underlying function $\map: Y \mapsto Z$, for some $Y \subset \R^d$, possibly for $d >> 1$ and $Z$ is a finite or even an infinite dimensional Banach space.

For many interesting systems in physics and engineering, the evaluation of the underlying map $\map$ requires the (approximate) solution of ordinary or partial differential equations. Prototypical examples for such $\map$'s include the response of an electric circuit to a change in input current, the change in bulk stresses on a structure on account of the variation of the load, the mean global sea surface temperatures that result from different levels of $CO_2$ emissions and the lift and the drag across the wing of an aircraft for different operating conditions such as mach number and angle of attack of the incident flow.

The cost of computing the underlying map $\map$ for different inputs can be very high, for instance when evaluating $\map$ requires computing solutions of PDEs in several space dimensions. Moreover, a large number of problems of interest are of the \emph{many query} type, i.e. many different instances of the underlying map have to be evaluated for different inputs. These problems can arise in the context of prediction (predict the wave height at a buoy when the tsunami is triggered by an earthquake with a certain initial wave displacement), uncertainty quantification (calculate statistics of $\map$ for uncertain inputs such as the mean and the variance of the stress due to random loading), optimal design and control (design the wing shape to minimize drag for constant lift) and (Bayesian) inverse problems (calibrate parameters in a climate model to match observed mean sea surface temperatures). The computational cost of such many query problems, requiring a large number of computational PDE solves, can be prohibitively expensive, even on state of the art high-performance computing (HPC) platforms.  

One paradigm for reducing the computational cost of such many query problems consists of generating \emph{training data}, i.e. computing $\map(y)$, $\forall y \in \train$, with $\train \subset Y$ denoting a \emph{training set}. These computations are done in an \emph{offline} mode. Then, a \emph{surrogate model} is constructed by designing a surrogate map, $\hat{\map}: Y \mapsto Z$ such that $\hat{\map}(y) \approx \map(y)$, for all $y \in \train$. Finally, in an \emph{online} step, one evaluates $\hat{\map}(y)$ with $y \in Y \setminus \train$ to perform the prediction. This surrogate model will be effective as long as $\map \approx \hat{\map}$ in a suitable sense and the cost of evaluating the surrogate map $\hat{\map}$ is significantly lower than the cost of evaluating the underlying map $\map$. Examples of such surrogate models include reduced order models \cite{ROMbook} and Gaussian process regression \cite{GPRbook}. 

A particularly attractive class of such surrogate models are \emph{deep neural networks} \cite{DLbook}, i.e. functions formed by concatenated compositions of affine transformations and scalar non-linear activation functions. Deep neural networks have been extremely successful at diverse tasks in science and engineering \cite{DL-nat} such as at image and text classification, computer vision, text and speech recognition, autonomous systems and robotics, game intelligence and even protein folding \cite{deepfold}. 

Deep neural networks are also being increasingly used in different contexts in scientific computing. A very incomplete sample of this rapidly growing field is in solutions of PDEs by so-called physics informed neural networks (PINNs) \cite{Lag1,KAR1,KAR2} and references therein, solutions of high-dimensional PDEs, particularly in finance \cite{E1,E2,E3} and improving the efficiency of existing numerical methods for PDEs, for instance in \cite{INC,DR1,DL_SM1} and references therein.

Almost all of the afore-mentioned papers in scientific computing use deep neural networks in the context of \emph{supervised learning} \cite{DLbook}, i.e. training the tuning parameters (weights and biases) of the neural network to minimize the so-called loss function (difference between the underlying map and the neural network surrogate on the training set in some suitable norm) with a stochastic gradient descent method. 

The accuracy of this supervised learning procedure is often expressed in terms of the so-called \emph{generalization error} or population risk \cite{MLbook} (see \eqref{eq:egen1} for the precise definition). The generalization error measures the error of the network on unseen data and it is notoriously hard to estimate it sharply \cite{AR1,NEYS1} and references therein. Nevertheless, it is customary to estimate it in terms of the following bound \cite{MLbook},
\begin{equation}
    \label{eq:genE1}
    \er_G \sim \er_T + \frac{U}{\sqrt{N}},
\end{equation}
with $\er_G$ being the generalization error. Here, $N$ is the number of training samples and $\er_T$ is the so-called \emph{training error} or \emph{empirical risk}, that is readily computed from the loss function (see \eqref{eq:etrain} for a definition).

When the training samples are chosen randomly from the underlying probability distribution on $Y$, one can think of the upper bound $U$ in  \eqref{eq:genE1} as a sort of \emph{standard deviation} of the underlying map $\map$ (see section 2.3, eqn (2.22) of a recent paper \cite{LMM1} for a precise estimate and discussion of this issue). Even with this bound \eqref{eq:genE1}, one immediately confronts the following challenges for using 
neural networks as surrogates in this context:
\begin{itemize}
    \item The bound \eqref{eq:genE1} stems from an application of the central limit theorem and indicates a slow decay of the error with respect to the number of training samples. In particular, as long as the standard deviation of the underlying map $\map$ is ${\mathcal O}(1)$, we require a large number of training samples to obtain low generalization errors. For instance, ${\mathcal O}(10^4)$ training samples are needed in order to obtain a generalization error of $1\%$ relative error. Computing such large number of training samples by evaluating the underlying PDEs is very expensive. 
    \item Furthermore, the upper bound $U \sim {\rm std}(\map)$ only holds under the assumption of (statistical) independence of the evaluations of the trained neural network. This is not necessarily true as the training process (stochastic gradient descent) can lead to \emph{strong correlations} between evaluations of the neural network, at different points of the training set. Estimating these correlations is the fundamental problem in obtaining sharp upper bounds on the generalization error \cite{AR1,EMW1} and tools from statistical learning theory \cite{CS1} such as the Vapnik-Chervonenkis (VC) dimension and Rademacher complexity \cite{MLbook} are often used for this bound. Consequently, the estimate \eqref{eq:genE1} can indicate an even slower decay of the error. 
\end{itemize}
As generating training data is expensive, the preceding discussion makes it clear that \emph{learning} maps, at least in the context of scientific computing, with randomly chosen training samples can be hard. This is in contrast to the \emph{big data} successes of deep learning in computer science where generating and accessing data (images, text, speech) is cheap. 

One remedy for this problem would be to use \emph{variance reduction procedures} to reduce the standard deviation of the underlying map in \eqref{eq:genE1}. A robust variance reduction method is the so-called multi-level procedure \cite{HEIN1,GIL1} and it was adapted to design a \emph{multi-level training algorithm} for deep neural networks in the recent paper \cite{LMM1}. However, this multi-level training procedure, in conjunction with randomly chosen training points, is efficient only when the amplitude of correlations in the training process is low (see \cite{LMM1}) and might not be suitable for many problems. 

Another possible solution to this problem of finding accurate neural network surrogates with a few training samples, was presented in a recent paper \cite{LMR1}, where the authors proposed selecting \emph{low-discrepancy sequences} as training sets for supervising learning with deep neural networks. Such low-discrepancy or \emph{quasi-random} sequences were developed in the context of Quasi-Monte Carlo (QMC) methods for numerical integration \cite{CAF1,DKS1} and references therein and many variants of these are available. The key point in using low-discrepancy sequences for numerical integration is the fact that the resulting quadrature errors decay linearly with respect to the number of quadrature points (up to logarithmic corrections), at least as long the underlying integration is over a domain in moderately high dimensions, resulting in a clear gain over the competing Monte Carlo algorithms.

The heuristic argument presented in \cite{LMR1} to motivate the use of low-discrepancy sequences to train deep neural networks was based on the \emph{equidistribution} property of these sequences, i.e. these sequences fill the underlying domain more uniformly than random points and hence can better represent the underlying map $\map$. Although the numerical results with such an algorithm as presented in \cite{LMR1}, \cite{LMM1} (also for a multi-level version of this algorithm) were very convincing, no rigorous analysis was provided to explain the results and justify the underlying intuition. 

The fundamental goal of the present paper is to describe and rigorously analyze this QMC type deep learning algorithm. To this end, in this paper we will
\begin{itemize}
    \item Formalize and present a deep learning algorithm, based on choosing low-discrepancy sequences as training sets for supervised learning in deep neural networks. 
    \item Analyze this algorithm by presenting upper bounds on the generalization error and using these upper bounds to constrain the choice of activation functions.
    \item Present a wide range of numerical experiments to illustrate the algorithm and the derived bounds as well as to demonstrate the advantage of the proposed algorithm over the standard supervised deep learning algorithm with randomly chosen training data. 
\end{itemize}
Our aim would be to convince the reader of the efficacy of the proposed deep learning algorithm for a large range of applications in scientific computing, particularly for problems in moderately high dimensions.

The rest of the paper is organized as follows: in section \ref{sec:2}, we present the proposed algorithm. The algorithm is analyzed and upper bounds on the generalization error are derived in section \ref{sec:3} and numerical experiments, illustrating the algorithm, are presented in section \ref{sec:4}. 

\section{The deep learning algorithm}
\label{sec:2}
\subsection{Problem setup}
We consider an underlying domain $Y \subset \R^d$ and for definiteness we let $Y = [0,1]^d$ be the unit cube in $d$ dimensions. Moreover, we also consider an underlying probability measure $\mu \in \Prob(Y)$ and further assume that this measure is absolutely continuous with respect to the $d$-dimensional Lebesgue measure and has a density given by $\bar{\mu} \in C(Y;\R_+)$ with 
$$
\int\limits_{Y} \bar{\mu}(y) dy = 1.
$$
Our objective is to find efficient surrogates for maps of the form,
\begin{equation}
\label{eq:map}
\map: Y \mapsto \R^m.
\end{equation}
Without loss of generality, we take $m=1$. 

A prototypical example for the map $\map$ comes from time-dependent \emph{parametric PDEs} of the form, 
\begin{equation}
\label{eq:ppde}
\begin{aligned}
\partial_t \U(t,x,y) &= L\left(y, \U, \nabla_x \U, \nabla^2_x \U, \ldots \right), \quad \forall~(t,x,y) \in [0,T] \times D(y) \times Y, \\
\U(0,x,y) &= \overline{\U}(x,y), \quad \forall~ (x,y) \in D(y) \times Y, \\
L_{b} \U(t,x,y) &= \U_b (t,x,y), \quad \forall ~(t,x,y) \in [0,T] \times \partial D(y) \times Y.
\end{aligned}
\end{equation}
Here, $Y$ is the underlying parameter space. The spatial domain is labeled as $y \rightarrow D(y) \subset \R^{d_s}$ and  $\U: [0,T] \times D \times Y  \rightarrow \R^M$ is the vector of unknowns. The differential operator $L$ is in a very generic form and can depend on the gradient and Hessian of $\U$, and possibly higher-order spatial derivatives. For instance, the heat equation as well as the Euler or Navier-Stokes equations of fluid dynamics are specific examples. Moreover, $L_b$ is a generic operator for imposing boundary conditions. The parametric nature of the PDE \eqref{eq:ppde}, represented by the parameter space $Y$, can stem from uncertainty quantification or Bayesian inversion problems where the parameter space models uncertainties in inputs for the underlying PDE. The parametric nature can also arise from optimal design, control and PDE constrained optimization problems with $Y$ being the design (control) space.

For the parameterized PDE \eqref{eq:ppde}, we consider the following generic form of \emph{observables},
\begin{equation}
\label{eq:obsp}
L_g(y,\U) := \int\limits_0^T\int\limits_{D(y)} \psi(x,t) g(\U(t,x,y)) dx dt, \quad {\rm for} ~\mu~{\rm a.e}~y \in Y.
\end{equation} 
Here, $\psi \in L^1_{{\rm loc}} (D(y) \times (0,T))$ is a  \emph{test function} and $g \in C^s(\R^M)$, for $s \geq 1$. 

For fixed functions $\psi,g$, we define the \emph{parameters to observable} map:
\begin{equation}
\label{eq:ptoob}
\map:y \in Y \rightarrow \map(y) = L_g(y,\U),
\end{equation}
with $L_g$ being defined by \eqref{eq:obsp}. In general, it is only possible to evaluate $\map$ up to numerical error, by approximating it with a suitable numerical method.  

We will design a deep learning algorithm for approximating the underlying map $\map$ \eqref{eq:map} in the following steps.

\subsection{Training set}
As is customary in supervised learning (\cite{DLbook} and references therein), we need to generate or obtain data to train the network. To this end, we fix $N \in \N$ and select a set of points $\EuScript{S} = \{y_i\}_{1 \leq i \leq N}$, with each $y_i \in Y$. It is standard that the points in the training set $\EuScript{S}$ are chosen randomly from the parameter space $Y$, independently and identically distributed with the measure $\mu$. However, in this paper we follow \cite{LMR1} and choose \emph{low-discrepancy sequences} for our training set. 

\subsection{Low-discrepancy sequences}
Consider any sequence of points $\left\{y_n\right\}_{1 \leq n \leq N}$ with $y_n \in Y = [0,1]^d$ for all $n$, and for any set $J \subset Y$ denote
\begin{equation}
    \label{eq:rnj}
    R_N(J) := \frac{1}{N} \# \left \{y_n \in J\right \} - {\rm meas}(J).
\end{equation}
Note that \eqref{eq:rnj} is the error of approximating the Lebesgue measure ${\rm meas}$ of the set $J \subset Y$ with the sum $\frac{1}{N} \sum\limits_{n=1}^N \chi_J (y_n)$, with $\chi$ denoting the indicator function of the underlying set. 

Instead of considering arbitrary subsets of $Y$, we focus on \emph{rectangular sets} \cite{CAF1}. To define a $d$-dimensional rectangle, we need only two antipodal vertices. In particular, we consider the following rectangular sets, 
\begin{equation*}
    E^{\ast} := \left\{J \subset Y: J~{\rm is~a~rectangular~set~with~one~vertex}~0~{\rm another~vertex}~z\in Y \right\}.
\end{equation*}
Then, we follow \cite{CAF1} and define the star-discrepancy or simply, discrepancy, of the sequence $\{y_n\}$ as,
\begin{equation}
    \label{eq:dstar}
    D^{\ast}_N = \sup\limits_{J \in E^{\ast}} |R_N (J)|.
\end{equation}
Roughly speaking, the discrepancy measures how well the sequence fills the underlying domain and consequently, how well is the measure of an arbitrary rectangular subset $J$ of $Y$ approximated by the sum $\frac{1}{N} \sum\limits_{n=1}^N \chi_J (y_n)$.  
We are now in a position to follow \cite{CAF1} and define low-discrepancy sequences below.
\begin{definition}
\label{def:lds}
{\bf Low-discrepancy sequence}: A sequence of points $\left\{y_n\right\}_{1 \leq n \leq N}$ with $y_n \in Y$ is termed as a low-discrepancy sequence if the following holds for its discrepancy $D_N^{\ast}$ \eqref{eq:dstar}:
\begin{equation}
    \label{eq:lds}
    D_N^{\ast} \leq C\frac{(\log N)^d}{N},
\end{equation}
with a constant $C$ that possibly depends on the dimension $d$ but is independent of $N$.
\end{definition}
It turns out that there are different types of low-discrepancy sequences, many of them are built from number theoretic considerations. The simplest low-discrepancy sequence is the one-dimensional van der corput sequence. Widely used low-discrepancy sequences (in arbitrary dimensions) are the Sobol \cite{sobol}, Halton \cite{hal}, Owen \cite{owen1} and Niederreiter \cite{nd1} sequences.  

It is also straightforward to see that for a sequence of randomly chosen points, we have
\begin{equation}
    \label{eq:rdisc1}
    D_N^{\ast} \sim {\mathcal O}\left(\frac{1}{\sqrt{N}} \right).
\end{equation}
Thus, randomly chosen points do not constitute a low-discrepancy sequence as defined in \ref{def:lds}. This provides a formalism to the intuition that a sequence of low-discrepancy points fill the underlying space more uniformly than random points.  

Once the training set $\train$ is chosen, we perform a set of simulations to obtain $\map (y)$, for all $y \in \train$. This might involve expensive simulations of the underlying PDE \eqref{eq:ppde}. 

 \begin{figure}[htbp]
\centering
\includegraphics[width=8cm]{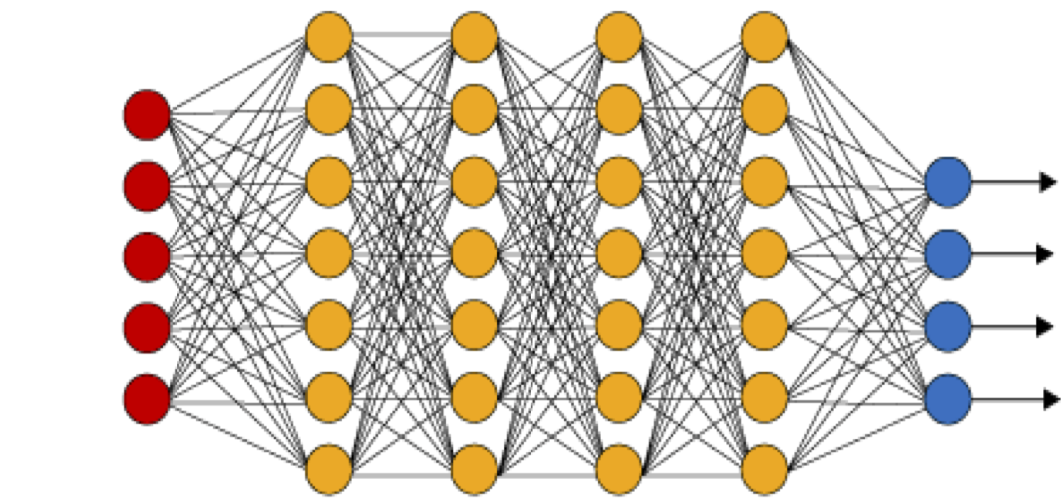}
\caption{An illustration of a (fully connected) deep neural network. The red neurons represent the inputs to the network and the blue neurons denote the output layer. They are
connected by hidden layers with yellow neurons. Each hidden unit (neuron) is connected by affine linear maps between units in different layers and then with nonlinear (scalar) activation functions within units.}
\label{fig:1}
\end{figure}

\subsection{Neural network} 
\label{sec:NN}
Given an input vector $y \in Y$, a feedforward neural network (also termed as a multi-layer perceptron), shown in figure \ref{fig:1}, transforms it to an output through layers of units (neurons) consisting of either affine-linear maps between units (in successive layers) or scalar non-linear activation functions within units \cite{DLbook}, resulting in the representation,
\begin{equation}
\label{eq:ann1}
\map_{\theta}(y) = C_K \circ\sigma \circ C_{K-1}\ldots \ldots \ldots \circ\sigma \circ C_2 \circ \sigma \circ C_1(y).
\end{equation} 
Here, $\circ$ refers to the composition of functions and $\sigma$ is a scalar (non-linear) activation function. A large variety of activation functions have been considered in the machine learning literature \cite{DLbook}. Popular choices for the activation function $\sigma$ in \eqref{eq:ann1} include the sigmoid function, the $\tanh$ function and the \emph{ReLU} function defined by,
\begin{equation}
\label{eq:relu}
\sigma(z) = \max(z,0).
\end{equation}
When, $z \in \R^p$ for some $p > 1$, then the output of the ReLU function in \eqref{eq:relu} is evaluated componentwise. 

For any $1 \leq k \leq K$, we define
\begin{equation}
\label{eq:C}
C_k z_k = W_k z_k + b_k, \quad {\rm for} ~ W_k \in \R^{d_{k+1} \times d_k}, z_k \in \R^{d_k}, b_k \in \R^{d_{k+1}}.
\end{equation}
For consistency of notation, we set $d_1 = d$ and $d_{K+1} = 1$. 

Thus in the terminology of machine learning (see also figure \ref{fig:1}), our neural network \eqref{eq:ann1} consists of an input layer, an output layer and $(K-1)$ hidden layers for some $1 < K \in \N$. The $k$-th hidden layer (with $d_{k+1}$ neurons) is given an input vector $z_k \in \R^{d_k}$ and transforms it first by an affine linear map $C_k$ \eqref{eq:C} and then by a ReLU (or another) nonlinear (component wise) activation $\sigma$ \eqref{eq:relu}. A straightforward addition shows that our network contains $\left(d + 1 + \sum\limits_{k=2}^{K} d_k\right)$ neurons. 
We also denote, 
\begin{equation}
\label{eq:theta}
\theta = \{W_k, b_k\}, \theta_W = \{ W_k \}\quad \forall~ 1 \leq k \leq K,
\end{equation} 
to be the concatenated set of (tunable) weights for our network. It is straightforward to check that $\theta \in \Theta \subset \R^M$ with
\begin{equation}
\label{eq:ns}
M = \sum\limits_{k=1}^{K} (d_k +1) d_{k+1}.
\end{equation}

\subsection{Loss functions and optimization} 
For any $y \in \train$, one can readily compute the output of the neural network $\map_{\theta} (y)$ for any weight vector $\theta \in \Theta$. We define the so-called training \emph{loss function} as 
\begin{equation}
\label{eq:lf1}
J (\theta) : = \sum\limits_{y \in \train} |\map(y) - \map_{\theta} (y) |^p \bar{\mu}(y),
\end{equation}
for some $1 \leq p < \infty$ and with $\bar{\mu}$ being the density of the underlying probability distribution $\mu$.   

The goal of the training process in machine learning is to find the weight vector $\theta \in \Theta$, for which the loss function \eqref{eq:lf1} is minimized. 

It is common in machine learning \cite{DLbook} to regularize the minimization problem for the loss function, i.e. we seek to find
\begin{equation}
\label{eq:lf2}
\theta^{\ast} = {\rm arg}\min\limits_{\theta \in \Theta} \left(J(\theta) + \lambda \reg(\theta) \right).
\end{equation}  
Here, $\reg:\Theta \to \R$ is a \emph{regularization} (penalization) term. A popular choice is to set  $\reg(\theta) = \|\theta_W\|^q_q$ for either $q=1$ (to induce sparsity) or $q=2$. The parameter $0 \leq \lambda << 1$ balances the regularization term with the actual loss $J$ \eqref{eq:lf1}. 

The above minimization problem amounts to finding a minimum of a possibly non-convex function over a subset of $\R^M$ for possibly very large $M$. We follow standard practice in machine learning by either (approximately) solving \eqref{eq:lf2} with a full-batch gradient descent algorithm or variants of mini-batch stochastic gradient descent (SGD) algorithms such as ADAM \cite{ADAM}. 

For notational simplicity, we denote the (approximate, local) minimum weight vector in \eqref{eq:lf2} as $\theta^{\ast}$ and the underlying deep neural network $\map^{\ast}= \map_{\theta^{\ast}}$ will be our neural network surrogate for the underlying map $\map$. 

The proposed algorithm for computing this neural network is summarized below.
\begin{algorithm} 
\label{alg:DL} {\bf Deep learning of parameters to observable map} 
\begin{itemize}
\item [{\bf Inputs}:] Underlying map $\map$ \eqref{eq:map} and low-discrepancy sequences in $Y$. 
\item [{\bf Goal}:] Find neural network $\map_{\theta^{\ast}}$ for approximating the underlying map $\map$. 
\item [{\bf Step $1$}:] Choose the training set $\train = \{y_n\}$ for $y_n \in Y$, for all $1 \leq n \leq N$ such that the sequence $\{y_n\}$ is a low-discrepancy sequence in the sense of definition \ref{def:lds}. Evaluate $\map(y)$ for all $y \in \train$ by a suitable numerical method. 
\item [{\bf Step $2$}:] For an initial value of the weight vector $\overline{\theta} \in \Theta$, evaluate the neural network $\map_{\overline{\theta}}$ \eqref{eq:ann1}, the loss function \eqref{eq:lf2} and its gradients to initialize the
(stochastic) gradient descent algorithm.
\item [{\bf Step $3$}:] Run a stochastic gradient descent algorithm till an approximate local minimum $\theta^{\ast}$ of \eqref{eq:lf2} is reached. The map $\map^{\ast} = \map_{\theta^{\ast}}$ is the desired neural network approximating the map $\map$.
\end{itemize}
\end{algorithm}
\section{Analysis of the deep learning algorithm \ref{alg:DL}}
\label{sec:3}
In this section, we will rigorously analyze the deep learning algorithm \ref{alg:DL} and derive bounds on its accuracy. To do so, we need the following preliminary notions of variation of functions.
\subsection{On Variation of functions}
\label{sec:31}
For ascertaining the accuracy of the deep learning algorithm \ref{alg:DL}, we need some hypotheses on the underlying map $\map$. As is customary in the analysis of Quasi-Monte Carlo (QMC) methods for numerical integration of functions \cite{CAF1}, we will express the regularity of the underlying map $\map$ in terms of its \emph{variation}. As long as $d=1$, the unambiguous notion of variation is given by the total variation of the underlying function. However, there are multiple notions of variations for multi-variate functions. In this section, we very closely follow the presentation of \cite{owen} to consider variation in the sense of Vitali and Hardy-Krause. 

To this end, we consider scalar valued functions defined on a hyperrectangle 
$[\mathbf{a},\mathbf{b}] = \{x \in \mathbb{R}^d \hspace{0.1cm}|\hspace{0.1cm} a_i \leq x_i \leq b_i \hspace{0.2cm} \forall i = 1,\dots,d\}$ and define ladders $\mathcal{Y} = \prod_{j=1}^d \mathcal{Y}^j$ on $[\mathbf{a},\mathbf{b}]$, 
where $\mathcal{Y}^j$ for each $j = 1,\dots,d$ is a set of finitely many values in $(a_j,b_j)$. 
Each $y^j \in \mathcal{Y}^j$ has a successor $y_+^j$, which is the smallest element in $(y^j,\infty)\cap\mathcal{Y}^j$ 
if $(y^j,\infty)\cap\mathcal{Y}^j \neq \emptyset$, and $b_j$ otherwise. Consequently, we define $\mathbf{y}_+$ to be the successor of $\mathbf{y} \in \mathcal{Y}$, where $(\mathbf{y}_+)_j$ is the succesor of $(\mathbf{y})_j$ for all $j=1,\dots,d$.
Furthermore, we define $\mathbb{Y}$ to be the set of all ladders on $[\mathbf{a},\mathbf{b}]$. 

\begin{definition}
The variation of a function $f$ on the hyperrectangle $[\mathbf{a},\mathbf{b}]$ in the sense of Vitali is given as
\begin{align*}
V(f) := V_{[\mathbf{a},\mathbf{b}]}(f) = \sup_{\mathcal{Y} \in \mathbb{Y}} \sum_{\mathbf{y} \in \mathcal{Y}} \sum_{v \subset 1:d} (-1)^{|v|}f(\mathbf{y}^v:\mathbf{y}_+^{-v}),
\end{align*}
where $\mathbf{y}^v:\mathbf{y}_+^{-v}$ is the glued vector $\hat{\mathbf{y}} \in [\mathbf{a},\mathbf{b}]$ with 
\begin{align*}
\hat{y}_i = 
\begin{cases}
y_i & i \in v \\
(y_+)_i& i \notin v
\end{cases}.
\end{align*}
Hence, for $\mathbf{x} \in [\mathbf{a},\mathbf{b}]$ the object $\mathbf{x}^{-v}$ denotes the complement with respect to $1:d$ for all $v \subset 1:d$.
\end{definition}
The Hardy-Krause variation can now be defined based on the Vitali Variation.
\begin{definition}
The variation of a function $f$ on the hyperrectangle $[\mathbf{a},\mathbf{b}]$ in the sense of Hardy-Krause is given as
\begin{equation}
\label{eq:HKV}
V_{HK}(f) = \sum_{u \subsetneq 1:d} V_{[\mathbf{a}^{-u},\mathbf{b}^{-u}]} (f(\mathbf{x}^{-u};\mathbf{b}^u)),
\end{equation}
where $f(\mathbf{x}^{-u};\mathbf{b}^u)$ denotes the function $f(\mathbf{x}^{-u}:\mathbf{b}^u)$ with $\mathbf{x}^{-u}$ is the argument and $b^u$ is a parameter.
\end{definition}
Following this definition it is clear that $V(f) < \infty$ must hold in order for $V_{HK}(f)<\infty$ to hold. 
Additionally, $V_{HK}$ is a semi-norm, as it vanishes for constant functions. However, it can 
be extended to a full norm by adjoining the case $u=1:d$ to the sum in \eqref{eq:HKV}. 

For notational simplicity, we will fix the hyperrectangle $[\mathbf{a},\mathbf{b}]$ to be the unit cube $Y = [0,1]^d$ in $\R^d$.
\begin{remark}
Given a function $f$, it is clearly impractical to compute its Hardy-Krause variation in terms of formula \eqref{eq:HKV}. On the other hand, as long as the underlying function $f$ is smooth enough, we can use the following straightforward bound for the Hardy-Krause variation \cite{owen}:
\begin{equation}
    \label{eq:VHK}
    V_{HK}(f) \leq \hat{V}_{HK} = \int\limits_{Y} \left| \frac{\partial^d f(y)}{\partial y_1\partial y_2\cdots \partial y_d} \right | dy + \sum\limits_{i=1}^d \hat{V}_{HK}(f^{(i)}_1).
\end{equation}
Here, $f_1^{(i)}$ is the restriction of the function $f$ to the boundary $y_i = 1$. Since this restriction is a function in $d-1$ dimensions, the above formula \eqref{eq:VHK} is understood as a recursion with the Hardy-Krause variation of a univariate function being given by its total variation:
\begin{equation}
    \label{eq:VHK1d}
    \hat{V}_{HK}(f)= \int_0^1 \left| \frac{d f}{dy}(y) dy \right |.
    \end{equation}
Note that the inequality in \eqref{eq:VHK} is an identity as long as all the mixed partial derivatives that appear in \eqref{eq:VHK} are continuous \cite{CAF1}.

\end{remark}
\subsection{Estimates on the generalization error of algorithm \ref{alg:DL}}
\label{sec:32}
We fix $p=1$ in the loss function \eqref{eq:lf1}. Our aim in this section is to derive bounds on the so-called \emph{generalization error} \cite{MLbook} of the trained neural network $\map^{\ast}$ \eqref{eq:ann1} generated by the deep learning algorithm \ref{alg:DL}, which is customarily defined by,
\begin{equation}
    \label{eq:egen1}
    \er_G= \er_{G} (\theta^{\ast};\train) := \int\limits_{Y} |\map(y) - \map^{\ast}(y)| \bar{\mu}(y) dy.
\end{equation}
Note that this generalization error depends explicitly on the training set $\train$ and on the parameters $\theta^{\ast}$ of the trained neural network as $\map^{\ast}$ depends on them. However, we will suppress this dependence for notational convenience. Also, for notational convenience, we will assume that the underlying probability distribution $\mu$ is \emph{uniform}, i.e. $\bar{\mu} \equiv 1$. All the estimates presented below can be  readily but tediously extended to the case of $\bar{\mu} \in C^d(Y)$ and we omit these derivations here.  

Thus, we consider the generalization error of the form:
\begin{equation}
    \label{eq:egen}
    \er_G= \er_{G} (\theta^{\ast};\train) := \int\limits_{Y} |\map(y) - \map^{\ast}(y)| dy.
\end{equation}
For each training set $\train = \{y_n\}$ with $1 \leq n \leq N$, the training process in algorithm \ref{alg:DL} amounts to minimize the so-called \emph{training error}:
\begin{equation}
\label{eq:etrain}
\er_T = \er_{T} (\theta^{\ast};\train) : = \frac{1}{N}\sum\limits_{i=1}^N |\map(y_n) - \map^{\ast}(y_n) |.
\end{equation}
The training error can be calculated from the loss function \eqref{eq:lf1}, a posteriori after the training is concluded. Note that the training error explicitly depends on the training set $\train$ and we suppress this dependence. 

Given these definitions, we wish to estimate the generalization error in terms of the computable training error. This amounts to estimating the so-called \emph{generalization gap}, i.e. the difference between the generalization and training error. We do so in the lemma below.
\begin{lemma}
\label{lem:1}
Let $\map^{\ast}$ be a deep neural network of the form \eqref{eq:ann1}, generated by the deep learning algorithm \ref{alg:DL}, with the training set $\train = \{y_n\}$, for $y_n \in Y$ with $1 \leq n \leq N$ being a low-discrepancy sequence in the sense of \eqref{eq:lds}. Then the generalization gap is estimated by,
\begin{equation}
    \label{eq:gg1}
    |\er_G - \er_T| \leq C \frac{V_{HK}\left(|\map - \map^{\ast}| \right)(\log N)^d}{N},
\end{equation}
with $V_{HK}$ being the Hardy-Krause variation defined in \eqref{eq:HKV}. 
\end{lemma}
By observing the definitions of the generalization error \eqref{eq:egen} and training error \eqref{eq:etrain}, we see that the training error $\er_T$ is exactly the Quasi-Monte Carlo approximation of the generalization error $\er_G$, with the underlying integrand being $|\map-\map^{\ast}|$. Therefore, we can apply the well-known Koksma-Hlawka inequality \cite{CAF1} to obtain,
\begin{align*}
    |\er_G - \er_T| &\leq V_{HK}\left(|\map-\map^{\ast}|\right) D^{\ast}_N \quad ({\rm by~Koksma-Hlawka~inequality}) \\
    &\leq C V_{HK}\left(|\map-\map^{\ast}|\right) \frac{(\log N)^d}{N}, \quad ({\rm by~} \eqref{eq:lds})
\end{align*}
which is the desired inequality \eqref{eq:gg1}. 

Thus, we need to assume that the Hardy-Krause variation of the underlying map $\map$ is bounded as well as show that the Hardy-Krause variation of the trained neural network is finite in order to bound the generalization error. We do so in the lemma below.
\begin{lemma}
\label{lem:2}
Assume that the underlying map $\map: Y \mapsto \R$ is such that 
\begin{equation}
    \label{eq:lem21}
    \hat{V}_{HK}(\map) < \infty.
\end{equation}
Assume that the neural network $\map^{\ast}$ is of the form \eqref{eq:ann1} with an activation function $\sigma: \R \mapsto \R$ with $\sigma \in C^d(\R)$. Then, for any given tolerance $\delta > 0$, the generalization error \eqref{eq:egen} with respect to the neural network $\map^{\ast}$, generated by the deep learning algorithm \ref{alg:DL}, can be estimated as,
\begin{equation}
    \label{eq:gg2}
    \er_G \leq \er_T + \frac{C (\log N)^d}{N} + 2\delta.
\end{equation}
Here, the constant $C$ depends on the tolerance $\delta$ but is independent of the number of training samples $N$.
\end{lemma}
\begin{proof}
To prove \eqref{eq:gg2}, for any given $\delta > 0$, we use the following auxiliary function, 
\begin{equation}
    \label{eq:rdel}
    \rho_{\delta} \in C^{d}(\R \mapsto \R_{+}):~ \||u| - \rho_{\delta}(u)\|_{L^{\infty}(\R)} \leq \delta.
\end{equation}
The existence of such a function $\rho_\delta$ can easily be verified by mollifying the absolute value function in the neighborhood of the origin. 

We denote,
$$
 \er^{\delta}_G:= \int\limits_{Y} \rho_{\delta}\left(\map(y) - \map^{\ast}(y)\right) dy, \quad
 \er^{\delta}_T = \frac{1}{N}\sum\limits_{i=1}^N \rho_{\delta}\left(\map(y_n) - \map^{\ast}(y_n)\right). $$
 Then, it is easy to see that $\er^{\delta}_T$ is the QMC approximation of $\er^{\delta}_G$ and we can apply the Koksma-Hlawska inequality as in the proof of lemma \ref{lem:1} to obtain that,
 \begin{equation}
     \label{eq:lem21}
     |\er^{\delta}_G - \er^{\delta}_T| \leq C \frac{V_{HK}\left(\rho_{\delta}\left(\map - \map^{\ast} \right)\right)(\log N)^d}{N}.
 \end{equation}
From the assumption that the activation function $\sigma \in C^d(\R)$ and structure of the artificial neural network \eqref{eq:ann1}, we see that for any $\theta \in \Theta$, the neural network $\map^{\ast}$ is a composition of affine (hence $C^d$) and sufficiently smooth $C^d$ functions, with each function in the composition being defined on compact subsets of $\R^{\bar{d}}$, for some $\bar{d}\geq 1$. Thus, for this composition of functions, we can directly apply Theorem 4 of \cite{BO1} to conclude that 
$$
V_{HK}(\map^{\ast}) \equiv \hat{V}_{HK}(\map^{\ast}) < \infty.
$$
Moreover, we have assumed that $\hat{V}_{HK}(\map) < \infty$. It is straightforward to check using the definition \eqref{eq:VHK} that 
\begin{equation}
    \label{eq:lem22}
    \hat{V}_{HK}(\map- \map^{\ast}) < \infty.
\end{equation}
Next, we observe that $\rho_{\delta} \in C^d(\R)$ by our assumption \eqref{eq:rdel}. Combining this with \eqref{eq:lem22}, we use Theorem 4 of \cite{BO1} (with an underlying multiplicative Faa di Bruno formula for compositions of multi-variate functions) to conclude that 
\begin{equation}
    \label{eq:lem23}
    V_{HK}\left(\rho_{\delta}\left(\map- \map^{\ast}\right)\right) \leq C < \infty.
\end{equation}
Here, the constant $C$ depends on the dimension and the tolerance $\delta$. 

Applying \eqref{eq:lem23} in \eqref{eq:lem21} and identifying all the constants as $C$ yields,
\begin{equation}
     \label{eq:lem24}
     |\er^{\delta}_G - \er^{\delta}_T| \leq C(d,\delta) \frac{(\log N)^d}{N}.
 \end{equation}
 It is straightforward to conclude from assumption \eqref{eq:rdel} that 
 \begin{equation}
     \label{eq:lem25}
     \max\left \{|\er_G - \er_G^{\delta}|, |\er_T - \er_T^{\delta}| \right \} \leq \delta.
 \end{equation}
Combining \eqref{eq:lem24} with \eqref{eq:lem25} and a straightforward application of the triangle inequality leads to the desired inequality \eqref{eq:gg2} on the generalization error. 
\end{proof}
Several remarks about Lemma \ref{lem:2} are in order.
\begin{remark}
The estimate \eqref{eq:gg2} requires a certain regularity on the underlying map $\map$, namely that $\hat{V}_{HK}(\map) < \infty$. This necessarily holds if the underlying map $\map \in C^d$. However, this requirement is significantly weaker than assuming that the underlying map is $d$-times continuously differentiable. Note that the condition of the Hardy-Krause variation is only on the mixed-partial derivatives and does not require any boundedness on other partial derivatives. Moreover, in many applications, for instance when the underlying parametric PDE \eqref{eq:ppde} is elliptic or parabolic and the functions $\psi,g$ in \eqref{eq:obsp} are smooth, this regularity is always observed. 
\end{remark}
\begin{remark}
It is instructive at this stage to compare the deep learning algorithm \ref{alg:DL} based on a low-discrepancy sequence training set with one based on randomly chosen training sets. As mentioned in the introduction, the best-case bound on the generalization error for randomly chosen training sets is given by \eqref{eq:genE1} with bound,
\begin{equation}
    \label{eq:genE2}
    \er_G \sim \er_T + \frac{{\rm std}(\map)}{\sqrt{N}},
    \end{equation}
with the above inequality holding in the root mean square sense (see estimate (2.22) and the discussion around it in the recent article \cite{LMM1}). Comparing such an estimate to \eqref{eq:gg2} leads to the following observations:
\begin{itemize}
    \item First, we see from \eqref{eq:gg2}, that as long as the number of training samples $N \geq 2^d$, the generalization error for the deep learning algorithm with a low-discrepancy sequence training set decays at much faster (linear) rate than the corresponding $1/\sqrt{N}$ decay for the deep learning algorithm with a randomly chosen training set. 
    \item As is well known \cite{MLbook,EMW1,AR1}, the bound \eqref{eq:genE2} on the generalization error needs not necessarily hold due to a high degree of correlations in the trained neural network when evaluated on training samples. Consequently, the bound \eqref{eq:genE2}, and the consequent rate of convergence of the deep learning algorithm with randomly chosen training set with respect to number of training samples can be significantly worse. No such issue arises for the bound \eqref{eq:gg2} on the generalization error for the deep learning algorithm \ref{alg:DL} with a low-discrepancy sequence training set, as this bound is based on the Koksma-Hlawka inequality and automatically takes correlations into account.
\end{itemize}
Given the above considerations, we can expect that at least till the dimension of the underlying map is moderately high and if the map is sufficiently regular, the deep learning algorithm \ref{alg:DL} with a low-discrepancy sequence training set will be \emph{significantly more accurate} than the standard deep learning algorithm with a randomly chosen training set. On the other hand, either for maps with low-regularity or for problems in very high dimensions, the theory developed here suggests no significant advantage using a low-discrepancy sequence as the training set over randomly chosen points, in the context of deep learning.  
\end{remark}
\begin{remark}
In principle, the upper bound \eqref{eq:gg2} on the generalization error is computable as long as good estimates on $\hat{V}_{HK}(\map)$ are available. In particular, the training error $\er_T$ \eqref{eq:etrain} is computed during the training process. The constants $C = C(d,\delta)$ in \eqref{eq:gg2} can be explicitly computed from the generalized Faa di Bruno formula (eqn [10] in reference \cite{BO1}) and given the explicit form \eqref{eq:ann1} of the neural network, we can compute $\hat{V}_{HK}(\map^{\ast})$ explicitly by computing the mixed partial derivatives. However, it is well known that the estimates on QMC integration error in terms of the Koksma-Hlawka inequality are very large overestimates \cite{CAF1}. Given this, we suspect that bounds of the form \eqref{eq:gg2} will be severe overestimates. On the other hand, the role of the bound \eqref{eq:gg2} is to demonstrate the linear decay of the error, with respect to the number of training samples, at least up to moderately high dimensions. 
\end{remark}
\begin{remark}
Finally, we investigate the issue of activation functions in this context. In deriving the bound \eqref{eq:gg2}, we assume that the activation function $\sigma$ in the neural network \eqref{eq:ann1} is sufficiently smooth, i.e. $\sigma \in C^d(\R)$. Thus, standard choices of activation functions such as sigmoid and $\tanh$ are admissible under this assumption. 

On the other hand, the ReLU function \eqref{eq:relu} is widely used in deep learning. Can we use ReLU also as an activation function in the context of the deep learning algorithm \ref{alg:DL}, with low-discrepancy sequences as training sets? A priori, the generalization error bound \eqref{eq:gg1} only requires that $|\map -\map^{\ast}|$ has bounded Hardy-Krause variation \eqref{eq:HKV}. Note that this expression does not require the existence of mixed partial derivatives, up to order $d$, of the neural network $\map^{\ast}$ and hence, of the underlying activation function. Thus, ReLU might very well be an admissible activation function in this context. 

However, in \cite{owen}, the author considers the following function for any non-negative integers $d,r$:
\begin{equation}
    \label{eq:owenf}
    f_{d,r}(x)= \left(\max \left\{\sum_{i=1}^d x_i - \frac{1}{2},0\right\}\right)^r,
\end{equation}
and proves using rather elementary arguments with a ladder $\mathcal{Y}_i=\{0,1/(2m),\dots,(m-1)/(2m)\}$, that as long as $d \geq r+2$, the function $f_{d,r}$ \eqref{eq:owenf} has infinite variation in the sense of Vitali, and hence infinite Hardy-Krause variation \eqref{eq:HKV}. 

This argument can be extended in the following: for $d \geq 3$, consider the following simple 2-layer ReLU network $\map^{\ast}_{d}: Y \subset \R^d \mapsto \R_+$:
\begin{equation}
    \label{eq:owenf1}
     \map^{\ast}_{d}(y)= \max \left\{\sum_{i=1}^d y_i - \frac{1}{2},0\right\}.
\end{equation}
Note that \eqref{eq:owenf1} is a neural network of the form \eqref{eq:ann1} with a single hidden layer, ReLU activation function, all weights being $1$ and bias is $-0.5$. We conclude from \cite{owen} (proposition 16) that the Hardy-Krause variation of the ReLU neural network \eqref{eq:owenf1} is infinite. Thus, using ReLU activation functions in \eqref{eq:ann1} might lead to a blow-up of the upper bound in the generalization error \eqref{eq:gg1} and a reduced rate of convergence of the error with respect to the number of training samples. 

However, it is worth pointing out that the bound \eqref{eq:gg1} can be a gross over-estimate and in practice, we might well have much better generalization error with ReLU activation functions than what \eqref{eq:gg1} suggests. To test this proposition, we consider the following simple numerical example. We let $Y = [0,1]^3$, i.e. $Y$ is a $3$-dimensional cube. The underlying map is given by $\map(y) = f_{3,4}(y)$, with $f_{3,4}$ defined from \eqref{eq:owenf}. Clearly, for this underlying map $\hat{V}_{HK}(\map) < \infty$. Thus, using a sigmoid activation function in \eqref{eq:ann1} and applying the generalization error estimate \eqref{eq:gg2} suggests a linear decay of the generalization error. On the other hand, \eqref{eq:gg2} does not directly apply if the underlying activation function is ReLU \eqref{eq:relu}. We approximate this underlying map $\map(y) = f_{3,4}(y)$ with the following three algorithms: 
\begin{itemize}
    \item The deep learning algorithm \ref{alg:DL} with Sobol sequences as the training set and sigmoid activation function. This configuration is referred to as $DL_{sob}$.
    \item The deep learning algorithm \ref{alg:DL} with Sobol sequences as the training set and ReLU activation function. This configuration is referred to as ReLU-$DL_{sob}$.
    \item The deep learning algorithm \ref{alg:DL} with Random points as the training set and sigmoid activation function. This configuration is referred to as $DL_{rand}$.
\end{itemize}
For each configuration, we train deep neural networks with architectures and hyperparameters specified in table \ref{tab:ensemble_params} with exactly $1$ hidden layer. We train $100$ of these one-hidden layer networks and present the average (over the ensemble) generalization error (approximated on a test set with $8192$ sobol points) for each of the above 3 configurations, for different numbers of training points, and present the results in figure \ref{fig:failed_relu}. As seen from the figure, the generalization error with the $DL_{sob}$ algorithm is significantly smaller than the other 2 configurations and decays superlinearly with respect to the number of training samples. On the other hand, there is no significant difference between the ReLU-$DL_{sob}$ and $DL_{rand}$ algorithms for this example. The generalization error decays much more slowly than for the $DL_{sob}$ algorithm. Thus, this example illustrates that it might not be advisable to use ReLU as an activation function in the deep learning algorithm \ref{alg:DL} with low-discrepancy sequences. Smooth activiation functions, such as sigmoid and $\tanh$ should be used instead. 
\end{remark}

\section{Numerical experiments}
\label{sec:4}
\subsection{Details of implementation}
We have implemented the deep learning algorithm \ref{alg:DL} in two versions:  the $DL_{sob}$ algorithm, with a Sobol sequence \cite{sobol} as the training set and the $DL_{rand}$ algorithm with a randomly chosen training set. Both versions are based on networks with sigmoid activation functions. The test set for each of the experiments below has a size of $8192$ sampling points, which sufficiently outnumbers the number of training samples, and is based on the Sobol sequence when testing the $DL_{sob}$ algorithm and are uniformly random drawn when testing the $DL_{rand}$ algorithm.

There are several hyperparameters to be chosen in order to specify the neural network \eqref{eq:ann1}. We will perform an ensemble training, on the lines proposed in \cite{LMR1}, to ascertain the best hyperparameters. The ensemble of hyperparameters is listed in table \ref{tab:ensemble_params}. The weights of the networks are initialized according to the Xavier initialization \cite{xavier}, which is standard for training networks with sigmoid (or tanh) activation functions. We also note that we choose $p=2$ in the loss function \eqref{eq:lf1} and add the widely-used weight decay regularization to our training procedure, which is based on an $L^2$-penalization of the weights in \eqref{eq:lf2}. Although the theory has been developed in the $L^1$-setting for low-discrepancy sequences, we did not find any significant difference between the results for $p=1$ or $p=2$ in the training loss function. 

If not stated otherwise, for all experiments, we present the generalization error (by approximating it with an error on a test set) for the best performing network, selected in the ensemble training, for the $DL_{sob}$ algorithm. On the other hand, as Monte Carlo methods are inherently statistical, the average over $100$ retrainings (random initializations of weights and biases) of the test error will be presented for the $DL_{rand}$-algorithm. 

The training is performed with the well-known ADAM optimizer \cite{ADAM} in full-batch mode for a maximum of $20000$ epochs (or training steps).

The experiments are implemented in the 
programming language Python using the open source machine learning framework PyTorch \cite{pytorch} to represent and train the neural network models. The scripts to perform the ensemble training as well as the retraining procedure together with all data sets used in the experiments can be downloaded from \href{https://github.com/tk-rusch/Deep_QMC_sampling}{\textit{https://github.com/tk-rusch/Deep\_QMC\_sampling}}.
Next, we present a set of numerical experiments that are selected to represent different areas in scientific computing. 
\begin{table}[htbp]
  \caption{Hyperparameters for the ensemble training}
  \label{tab:ensemble_params}
  \centering
  \begin{tabular}{ll}
    \toprule
    \cmidrule(r){1-2}
    Hyperparameter     & values used for the ensemble \\
    \midrule
learning rate & $10^{-1},10^{-2},10^{-3}$\\
weight decay &  $10^{-4},10^{-5},10^{-6},10^{-7} $\\
depth & $2^{2},2^{3},2^4$\\
width & $3\times2^1,3\times2^2,3\times2^3$\\
    \bottomrule
  \end{tabular}
\end{table}
\subsection{Function approximation: sum of sines}
We start with an experiment that was proposed in \cite{LMR1} to test the abilities of neural networks to approximate nonlinear maps in moderately high dimension. The map to be approximated is the following sum of sines:
\begin{align}
\label{eq:sum_sines}
\map(x) = \sum_{i=1}^6 \sin(4\pi x_i),
\end{align}  
where $x \in [0,1]^6$. Clearly $\map$ is infinitely differentiable and satisfies the assumptions of Lemma \ref{lem:2}. On the other hand, given that we are in $6$ dimensions and that the derivatives of $\map$ are large, it is quite challenging to approximate this map by neural networks \cite{LMR1}. 

In fact the experiments in \cite{LMR1} show that the standard $DL_{rand}$ algorithm produces large errors even for a very large number of training samples due to a high training error. 

Motivated by this, we modify the deep learning algorithm slightly by introducing the well-known practice of \emph{batch normalization} \cite{batchnorm} to each hidden layer of the neural network. Using this modification of our network architectures, the training errors with the resulting algorithm, plotted in figure \ref{fig:scaled_sine_mean}, are now sufficiently small. Moreover, the generalization error for the $DL_{rand}$ algorithm decays at a rate of $0.63$ with respect to the number of training samples. This is slightly better than the predicted rate of $0.5$.

\begin{figure}[h!]
\begin{minipage}[t]{.49\textwidth}
\includegraphics[width=1.\textwidth]{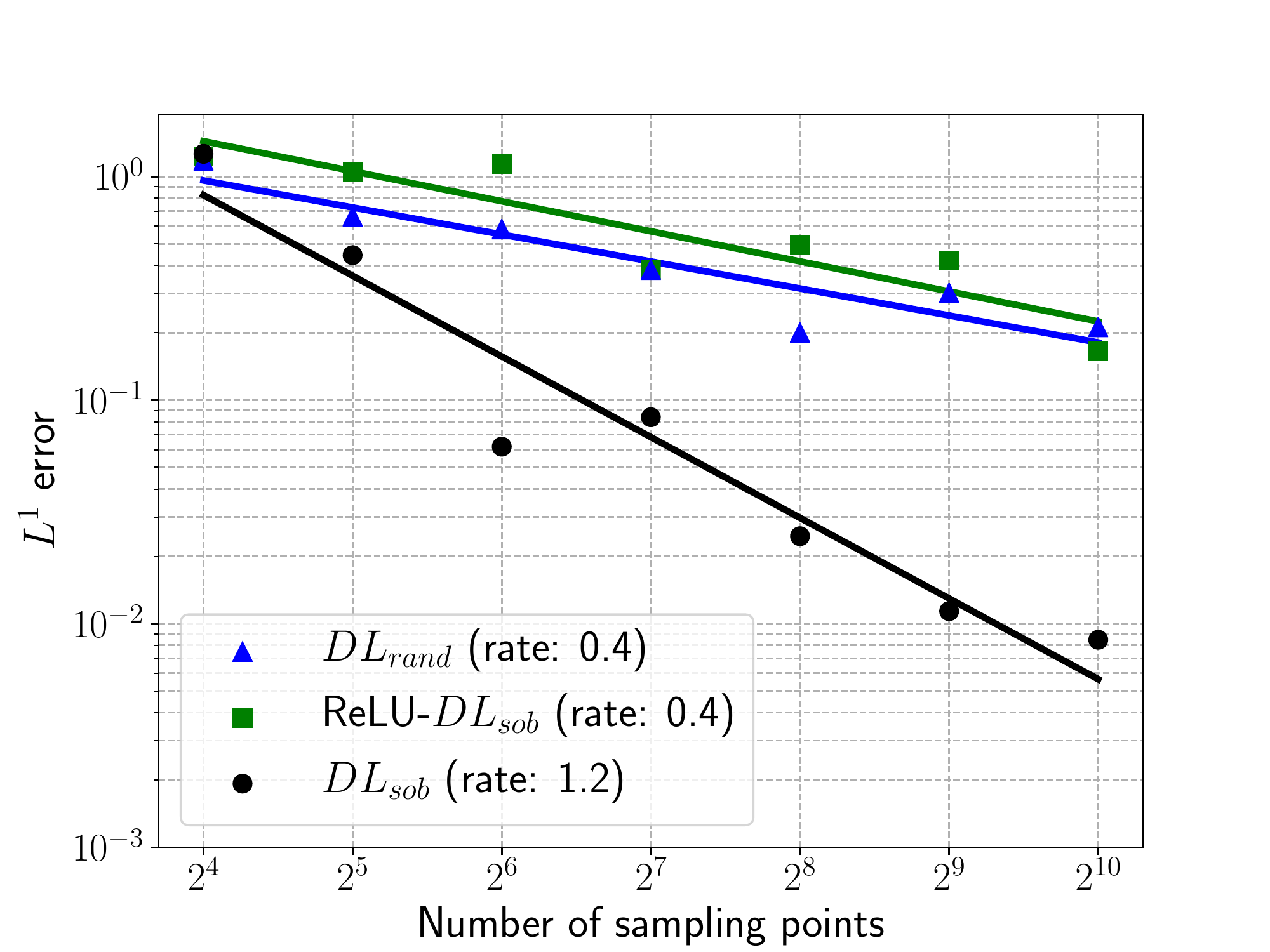}
\caption{$L^1$ generalization errors for approximating $f_{3,4}$ \eqref{eq:owenf} using $DL_{sob}$ based on ReLU and sigmoid activation functions compared to the generalization error when using $DL_{rand}$.}
\label{fig:failed_relu}
\end{minipage}%
\hspace{0.01\textwidth}
\begin{minipage}[t]{0.49\textwidth}	
\includegraphics[width=1.\textwidth]{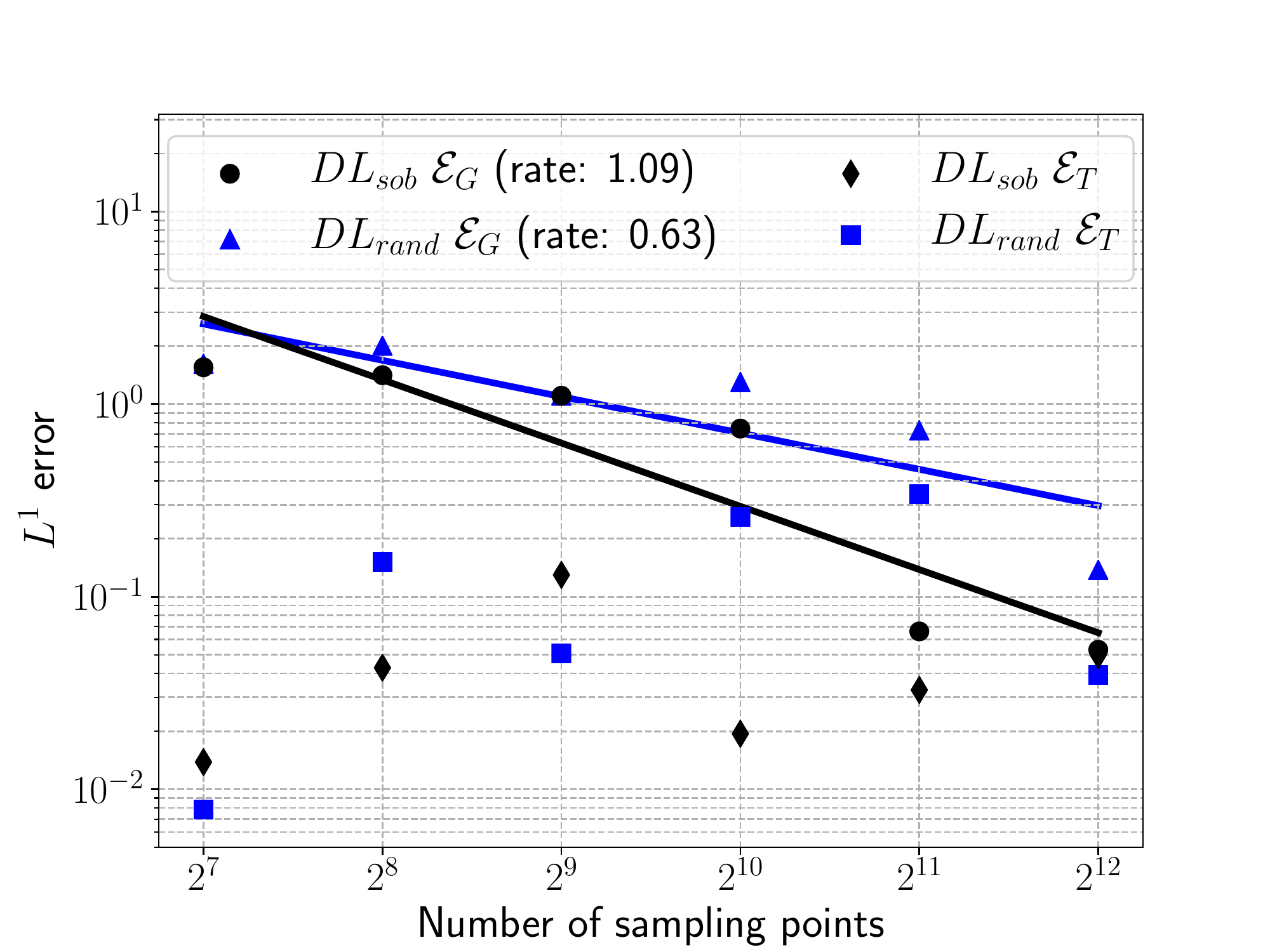}
\caption{$L^1$ generalization and training errors for approximating the sum of sines \eqref{eq:sum_sines} underlying mapping using $DL_{sob}$ as well as $DL_{rand}$.}
\label{fig:scaled_sine_mean}
\end{minipage}
\end{figure}

On the other hand, the $DL_{sob}$ algorithm, based on low-discrepancy Sobol points does much better. Indeed, the training error with $DL_{sob}$ is lower than the training error with the $DL_{rand}$ algorithm with most training sets, but not in all of them. However, the generalization gap (and hence the generalization error) of $DL_{sob}$ is significantly smaller than that of $DL_{rand}$ and decays at the rate of $1.1$, which is slightly better than the expected rate of $1$. This experiment clearly validates the theory developed earlier in the paper. Indeed comparing \eqref{eq:genE2} with \eqref{eq:gg2} shows that even if the differences in training errors might be minor (we have no control on the size of this error in \eqref{eq:gg2}), the generalization gap with the $DL_{sob}$ algorithm should be significantly smaller than the $DL_{rand}$ algorithm. This is indeed observed in the numerical results presented in figure \ref{fig:scaled_sine_mean}. 
\subsection{UQ for ODEs: projectile motion.}
\label{sec:42}
This problem is from \cite{LMM1} where it was introduced as a prototype for learning observables for ODEs and dynamical systems. The motion of a projectile, subjected to gravity and air drag, is described by the following system of ODEs,
\begin{alignat}{2}
\label{eq:projectile_system}
\frac{d}{dt} \mathbf{x}(t;y) &= \mathbf{v}(t;y), && \mathbf{x}(0;y)=\mathbf{x}_0(y), \\
\frac{d}{dt} \mathbf{v}(t;y) &= -F_D(\mathbf{v}(t;y);y)\mathbf{e}_1 - g\mathbf{e}_2, \qquad && \mathbf{v}(0;y) = \mathbf{v}_0(y),
\end{alignat}
where $F_D = \frac{1}{2m}\rho C_d \pi r^2 \|v\|^2$ denotes the drag force. 
Additionally, the initial conditions are set to $\mathbf{x}(0;y)= [0,h]$ and 
$\mathbf{v}(0;y) =[v_0 \cos(\alpha), v_0 \sin(\alpha)]$. 

In the context of uncertainty quantification (UQ) \cite{UQbook}, one models uncertainty in the system on account of measurement errors, in a statistical manner, leading to a parametric model of the form,
\begin{align*}
\rho(y) &= 1.225(1+ \epsilon G_1(y)),\hspace{0.1cm} r(y) = 0.23(1+ \epsilon G_2(y)),\hspace{0.1cm} C_D(y) = 0.1(1 + \epsilon G_3(y)), \\
m(y) &= 0.145(1 + \epsilon G_4(y)),\hspace{0.1cm} h(y) = (1 + \epsilon G_5(y)), \hspace{0.1cm} \alpha(y) = 30(1 + \epsilon G_6(y)), \hspace{0.1cm} v_0(y) = 25(1 + \epsilon G_7(y)),
\end{align*}
where $y \in Y=[0,1]^7$ is the parameter space, with an underlying uniform distribution. Additionally, $G_k(y) = 2y_k-1$ for all $k=1,\dots,7$ and $\epsilon = 0.1$.

We choose the \emph{horizontal range} $x_{\max}$ to be the observable of the simulation:
\begin{equation*}
    x_{\max}(y) = x_1(y,t_f), \qquad \text{with } t_f = x_2^{-1}(0).
\end{equation*}
The objective is to predict and approximate the map $\map = x_{max}$ with neural networks. To this end, we generate training and test data with a forward Euler discretization of the system 
\eqref{eq:projectile_system}. The whole solver can be downloaded from 
\href{https://github.com/mroberto166/MultilevelMachineLearning}{\textit{https://github.com/mroberto166/MultilevelMachineLearning}}.
\begin{figure}[h!]
\begin{minipage}[t]{0.49\textwidth}	
\includegraphics[width=1.\textwidth]{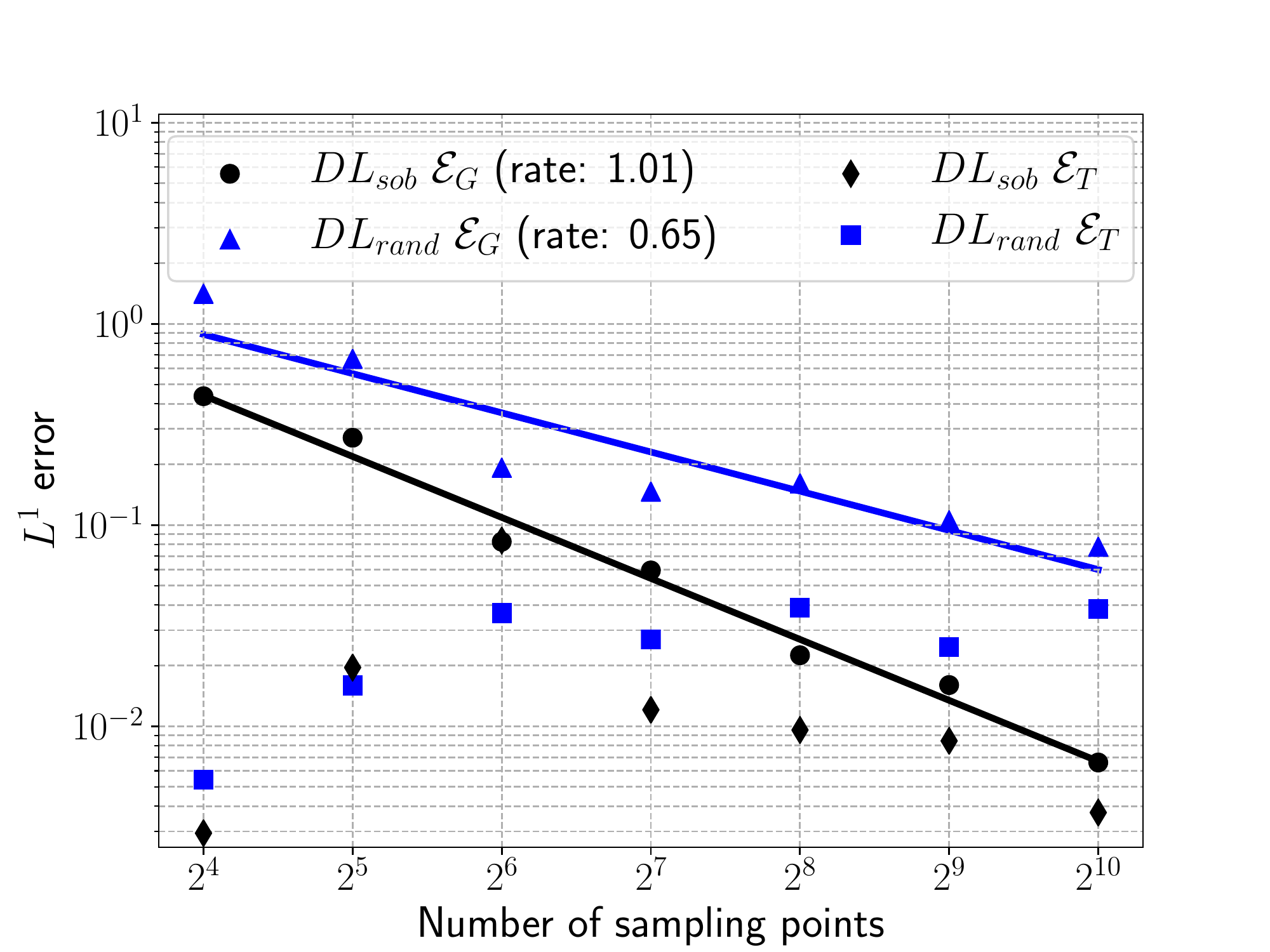}
\caption{$L^1$ generalization and training errors for approximating the horizontal range $x_{max}$ corresponding to the ODE system \eqref{eq:projectile_system} using $DL_{sob}$ as well as $DL_{rand}$. For the $DL_{sob}$ the best result of the ensemble is reported while for the $DL_{rand}$ the retrained average of the best ensemble result is shown.}
\label{fig:projectile_motion}
\end{minipage}%
\hspace{0.01\textwidth}
\begin{minipage}[t]{0.49\textwidth}	
\includegraphics[width=1.\textwidth]{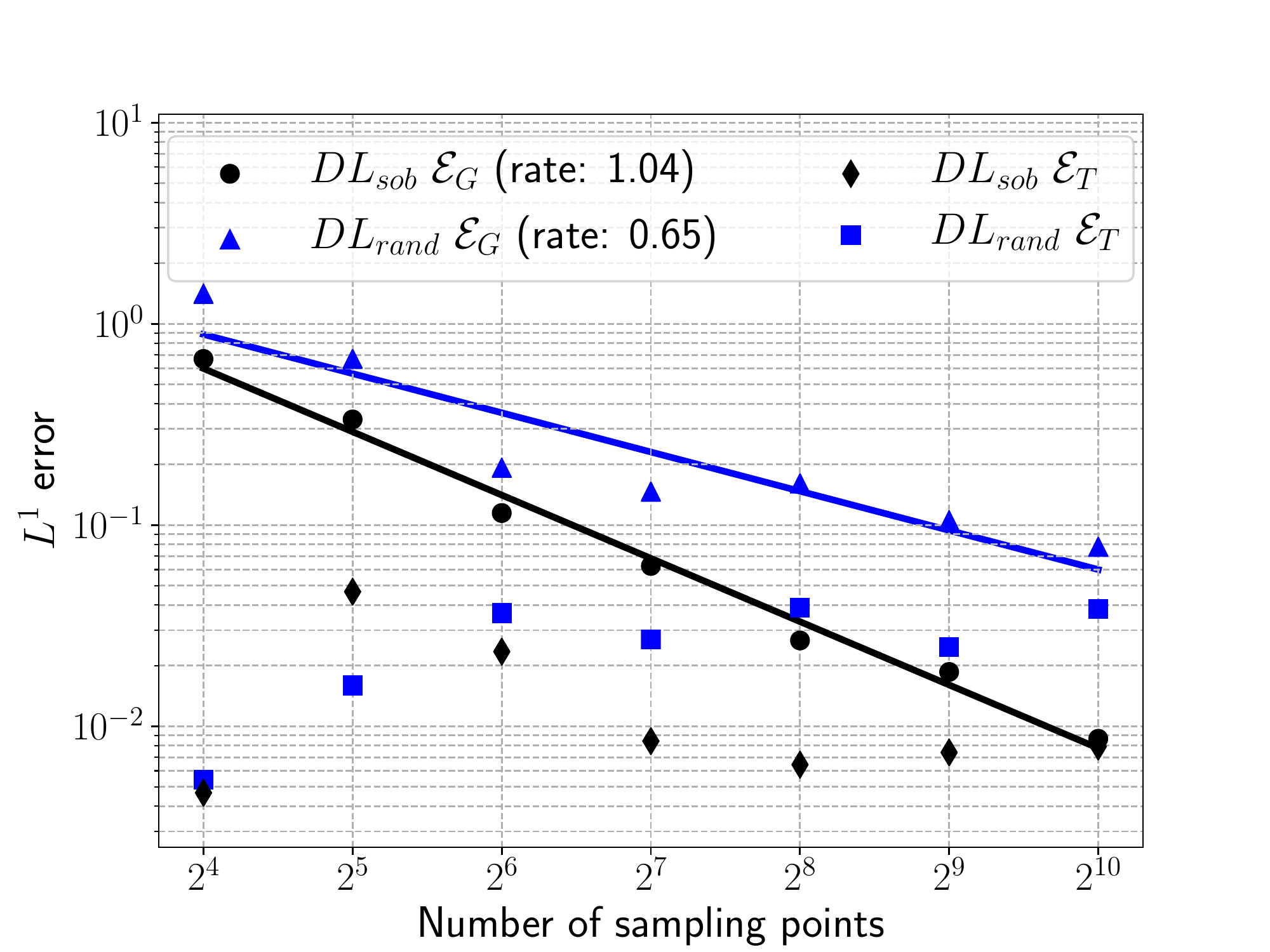}
\caption{$L^1$ generalization and training errors for approximating the horizontal range $x_{max}$ corresponding to the ODE system \eqref{eq:projectile_system} using $DL_{sob}$ as well as $DL_{rand}$. For both the $DL_{sob}$ as well as the $DL_{rand}$ the retrained average of the best ensemble result is reported.}
\label{fig:projectile_motion_mean}
\end{minipage}
\end{figure}
In the 
following experiment, the data sets are generated with a step size of  
$\Delta t = 0.00125$.
In figure \ref{fig:projectile_motion}, we present the results for the training error and the generalization error with both the $DL_{sob}$ as well as $DL_{rand}$ algorithms, with varying size of the training set. From this figure, we see that the training errors are smaller for the $DL_{sob}$ than for $DL_{rand}$. This is not explained by the bounds \eqref{eq:genE2} and \eqref{eq:gg2}, as we do not estimate the training error in any particular way. On the other hand, the generalization gap (and the generalization error) for $DL_{sob}$ is considerably smaller than the generalization error for $DL_{rand}$ and decays at approximately the expected rate of $1$, with respect to the number of training samples. This is in contrast to the generalization error for the $DL_{rand}$ which decays with a rate of $0.65$.

We remind the reader that the $DL_{sob}$ errors are for the best performing network selected in the ensemble training, whereas the $DL_{rand}$ errors are averages over the multiple retrainings. Could this difference in evaluating errors, which is completely justified by the theory presented here, play a role in how the errors behave? To test this, in figure \ref{fig:projectile_motion_mean}, we compare the training and generalization errors for $DL_{sob}$ and $DL_{rand}$, with both being averaged over $100$ retrainings. There is a slight difference in the average training error for $DL_{sob}$, when compared to the ones shown in figure \ref{fig:projectile_motion}. In particular, the average training error does not decay as neatly as for a single network. On the other hand, the generalization error decays at rate $1$ and is significantly smaller than the generalization error for $DL_{rand}$
\subsection{Computational Finance: pricing for a European Call Option} 
In this experiment, we want to compute the "fair" price of a \emph{European multi-asset basket call option}. We assume that there is a basket of underlying assets $S_1,\dots,S_d$, which change in time according to 
the multivariate geometric brownian motion (GBM), i.e.
\begin{align*}
dS_i &= rS_i(t)dt + \sigma_iS_i(t) dW_i(t), \\
S_i(0) &= (S_0)_i,
\end{align*}
for all $i=1,\dots,d$,
where $r$ is called the risk-free interest rate, $\sigma_i$ is the volatility and the Wiener processes $W_i$ are correlated 
such that $d\langle W_i,W_j\rangle(t) = \rho_{ij}dt$ for all $i,j=1,\dots,d$. 
Additionally, the payoff function associated to our basket option is given as 
\begin{align*}
\Lambda(\mathbf{S}) = \Lambda(S_1,\dots,S_d).
\end{align*}
The pricing function $V(t,\mathbf{S})$ is then the discounted conditional 
expectation of its payoff under an equivalent martingale measure $\tilde{P}$:
\begin{align*}
V(t,\mathbf{S}) = e^{-r(T-t)}\mathbb{E}_{\tilde{P}}[\Lambda(\mathbf{S}(T)) \hspace{0.1cm}|\hspace{0.1cm} \mathbf{S}(t) = \mathbf{S}].
\end{align*}
However, using the Feynman-Kac Formula, one can show that the payoff satisfies a multidimensional Black-Scholes partial differential equation (BSPDE) (\cite{bs_theory}) of the form: 
\begin{align}
\label{eq:bspde}
\frac{\partial V}{\partial t} + \sum_{i=1}^d r S_i \frac{\partial V}{\partial S_i} + \frac{1}{2}\sum_{i=1}^d \sigma_i^2 S_i^2\frac{\partial^2 V}{\partial S_i^2}
+\sum_{i=1}^{d-1} \sum_{j=i+1}^d \rho_{ij}\sigma_i \sigma_j S_i S_j \frac{\partial^2 V}{\partial S_i \partial S_j} -rV =0,
\end{align}
subject to the terminal condition $V(T,\mathbf{S}) = \Lambda(\mathbf{S})$. \\
Even though a general closed-form solution does not exist for the multi-dimensional 
BSPDEs, using a specific payoff function enables computing the solution in terms of an explicit formula. Therefore, we will consider the payoff function corresponding to the European geometric average basket call option:
\begin{align}
\label{eq:geom:payoff}
\Lambda(\mathbf{S}(T)) = \max \{ \left(\prod_{i=1}^d S_i(T)\right)^{1/d}-K,0\}.
\end{align}
Using the fact that products of log-normal random variables are log-normal, 
one can derive the following for the initial value of the pricing function, 
i.e. for the discounted expectation at time $t=0$:
\begin{align}
\label{eq:bs_initial}
V(0,\mathbf{S})=e^{-rT}\mathbb{E}_{\tilde{P}}[\Lambda(\mathbf{S}(T)) \hspace{0.1cm}|\hspace{0.1cm} \mathbf{S}(0) = \mathbf{S}].
\end{align}
This is of particular interest, as it denotes the value of the option at the time 
when it is bought. This provides the fair price for the European geometric average 
basket Call option in the Black-Scholes model. Using the particular payoff function 
\eqref{eq:geom:payoff}, the closed-form solution to the multidimensional BSPDE 
\eqref{eq:bspde} at point $(0,\mathbf{S})$ is then given as (Theorem 2 in \cite{cfs_bspde}):
\begin{equation}
    \label{eq:BSES}
V(0,\mathbf{S}) = e^{-rT}(\tilde{s}e^{\tilde{m}}\Phi({d}_1) - K\Phi({d}_2)),
\end{equation}
where $\Phi$ is the standard normal cumulative distribution function. Additionally
\begin{align*}
&\nu = \frac{1}{d} \sqrt{\sum_{j=1}^d \left(\sum_{i=1}^d \sigma_{ij}^2\right)^2}, \qquad m = rT - \frac{1}{2d}\sum_{i=1}^d\sum_{j=1}^d \sigma_{ij}^2T, \qquad \tilde{m} = m + \frac{1}{2}\nu^2,\\
&\tilde{s} = \left(\prod_{i=1}^d S_i\right)^{1/d}, \qquad d_1 = \frac{\log(\frac{\tilde{s}}{K})+m+\nu^2}{\nu}, \qquad d_2 = d_1 - \nu,
\end{align*}
where the entries $\sigma_{ij}$ of $\sigma \in \mathbb{R}^{d \times d}$ are the covariances 
(and variances) of the stock returns. For the subsequent experiments, we fix 
the following parameters: $\sigma = 10^{-5}\mathbb{1}$, $T=5$, $K=0.08$ and $r=0.05$,
where $\mathbb{1} \in \mathbb{R}^{d \times d}$ is the identity matrix. 

The underlying asset prices $S_i$, at time $t=0$ are drawn from $[0,1]^d$. The underlying map for this simulation is given by  
the fair price $\map({\bf S}) =  V(0,\mathbf{S})$ of the basket option at time $t=0$. We generate the training data for both the $DL_{sob}$ and the $DL_{rand}$ algorithms by computing $\map({\bf S})$ for either Sobol points or random points in $Y$, i.e. by evaluating the explicit formula \eqref{eq:BSES} for each point in the training set. We note that this formula is an explicit solution of the underlying Black-Scholes PDE \eqref{eq:bspde}. In case no formula is available, for instance with a different payoff function, either a finite difference or finite element method approximating the underlying PDE (in low dimensions) or a Multi-Carlo approximation of the integral \eqref{eq:bs_initial} can be used to generate the training data. 

In figure \ref{fig:multi_dim_BS}, we plot the generalization error for the $DL_{sob}$ with respect to the number of training points, till $2^{10}$ training points, for three different values of the underlying dimension $d$ (number of assets in the basket). For concreteness, we consider $d=5,7,9$ and see that the rate of decay of the generalization error does depend on the underlying dimension and ranges from $0.9$ (for $d=5$) to approximately $0.5$ for $d=9$. This is completely consistent with the bound \eqref{eq:gg2} on the generalization gap as the effect of the logarithmic correction is stronger for higher dimensions. In fact, the logarithmic term is expected to dominate till at least $N \approx 2^d$ training points. Once we consider even more training points, the logarithmic term is weaker and a higher rate of decay can be expected, as predicted by \eqref{eq:gg2}. To test this, we plot the generalization error till $2^{12}$ training points for the $9$-dimensional case ($d=9$ in \eqref{eq:bspde}), in figure \ref{fig:BS_QMC_vs_MC}. For the sake of comparison, we also plot the generalization errors with the $DL_{rand}$ algorithm. From figure \ref{fig:BS_QMC_vs_MC}, we observe that indeed, the rate of decay of generalization errors for the $DL_{sob}$ algorithm does increase to approximately $0.7$, when the 
number of training points is greater than $2^9$. In this regime, the $DL_{sob}$ algorithm significantly outperforms the $DL_{rand}$ algorithm providing a factor of $5$ speed up and a significantly better convergence rate, even for this moderately high-dimensional problem.  
\begin{figure}[h!]
\begin{minipage}[t]{0.49\textwidth}	
\includegraphics[width=1.\textwidth]{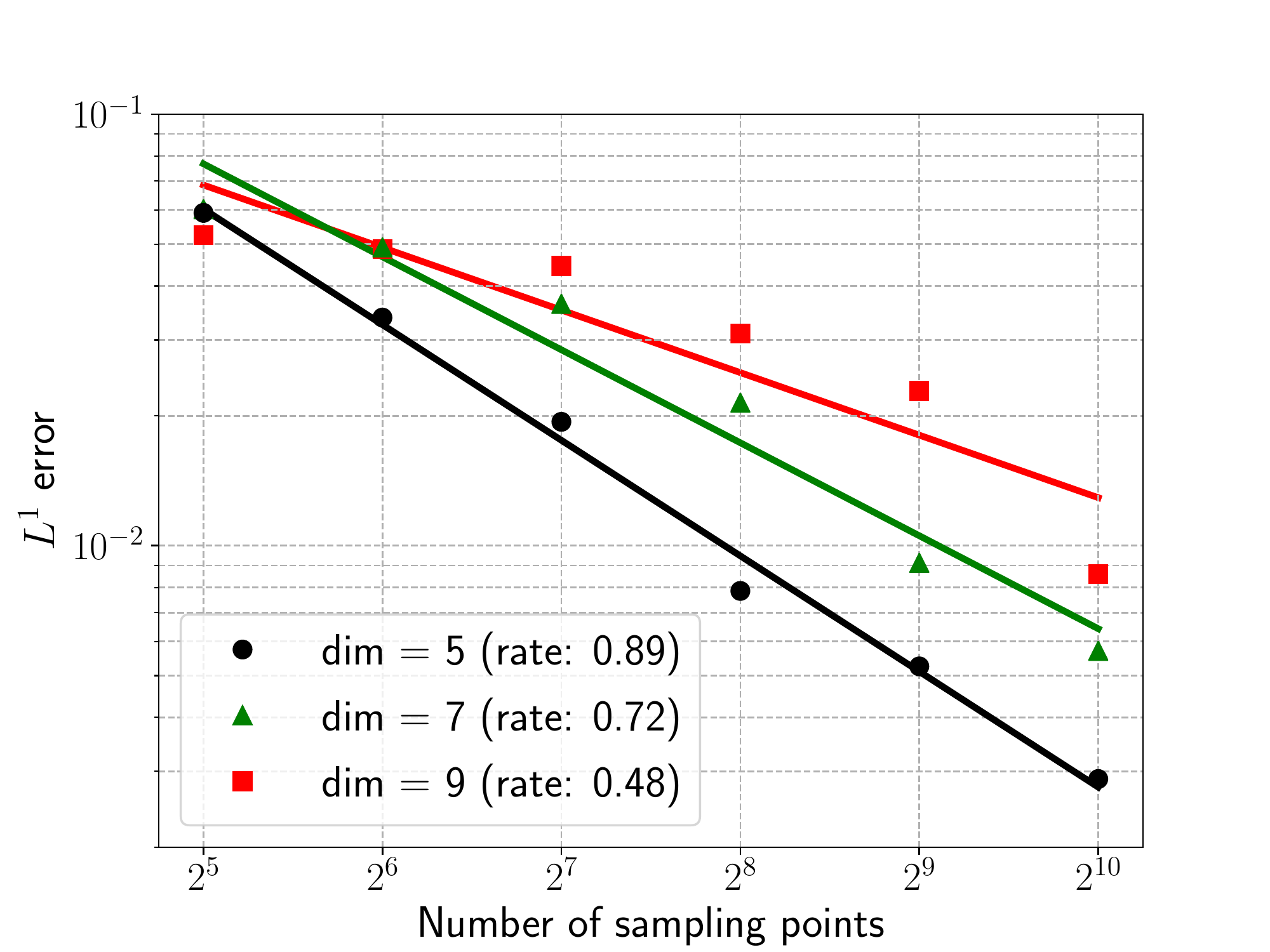}
\caption{$L^1$ generalization errors for approximating the fair price \eqref{eq:BSES} using $DL_{sob}$ for different input dimensions of the underlying map.}
\label{fig:multi_dim_BS}
\end{minipage}%
\hspace{0.01\textwidth}
\begin{minipage}[t]{0.49\textwidth}	
\includegraphics[width=1.\textwidth]{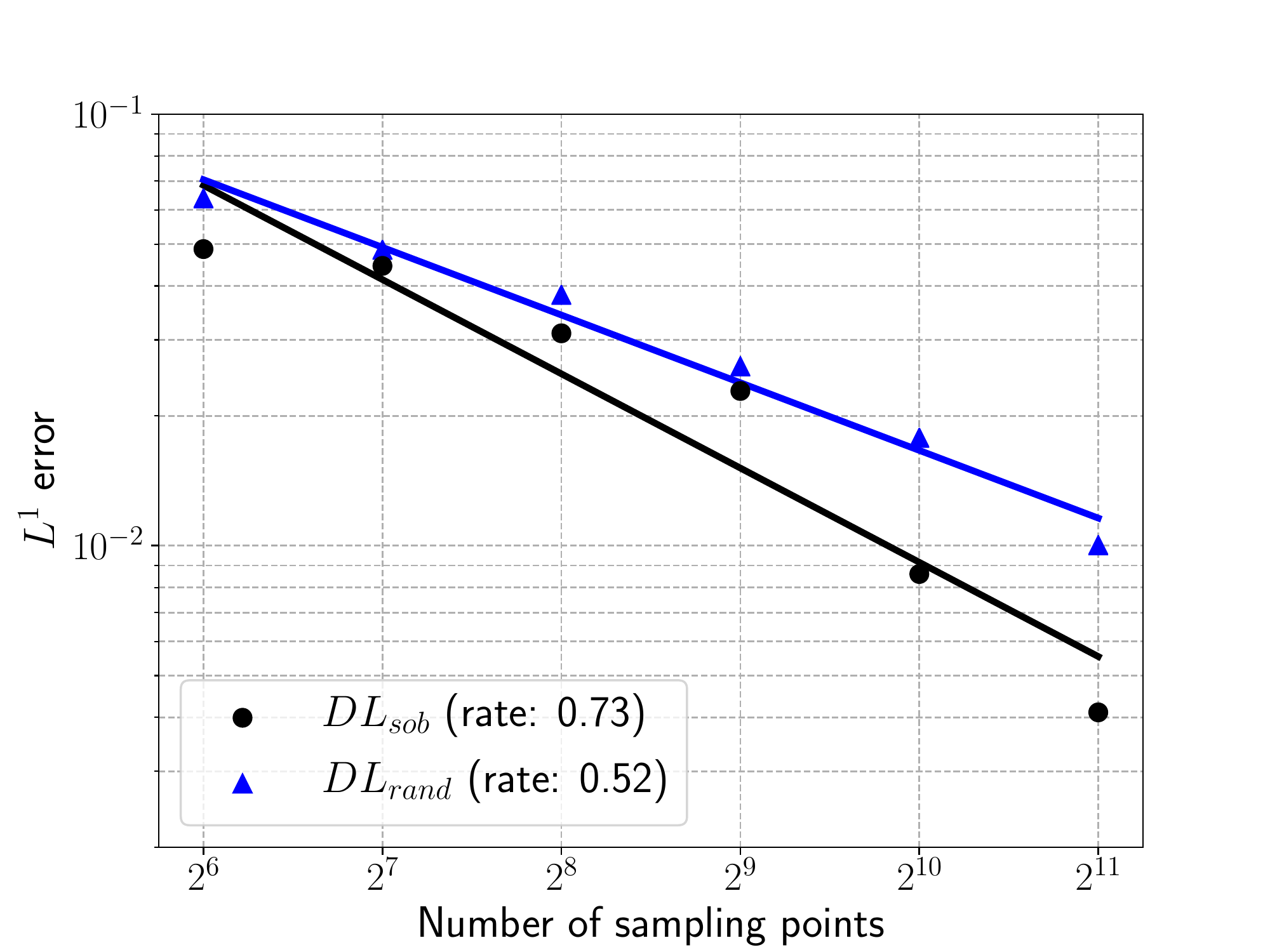}
\caption{$L^1$ generalization error for approximating the fair price \eqref{eq:BSES} for a fixed input dimensions of $9$ using $DL_{sob}$ compared to using $DL_{rand}$.}
\label{fig:BS_QMC_vs_MC}
\end{minipage}
\end{figure}
\subsection{Computational Fluid Dynamics: flow past airfoils}
In this section, we consider a prototypical example from computational fluid dynamics, namely that of a compressible flow past the well-studied RAE2822 airfoil \cite{UMRIDA}. The underlying parametric PDE \eqref{eq:ppde} are the well-known compressible Euler equations of aerodynamics (Eqns (5.1) of \cite{LMR1}) in two space dimensions, which are considered in an exterior domain of a \emph{perturbed} airfoil, parameterized by,
\begin{equation}
    \label{eq:afoil1}
    S(x;y)= \overline{S}(x) +  \sum\limits_{i=1}^d a_i \sin^{4}\left(\pi x^{\frac{\ln(0.5)}{\ln(x_{M_i})}}\right).
\end{equation}
Here, $\overline{S}:= [\overline{S}_L,\overline{S}_U]$ is the shape function for the reference RAE2822 airfoil (see figure \ref{fig:afoil}(left)) and the airfoil is deformed (perturbed) by the well-known \emph{Hicks-Henne bump functions} \cite{HH}, further parameterized by $a_i,x_{M_i}$ representing the amplitude and the location (maximum) of the perturbed basis functions, respectively.

The flow is defined by the following free-stream boundary conditions,
\begin{equation*}
    T^\infty = 1, \qquad M^\infty = 0.729, \qquad
    p^\infty = 1, \qquad \alpha = 2.31^{\circ},
\end{equation*}
where $\alpha$ denotes the angle of attack and $T,p,M$ are the temperature, pressure and Mach-number of the incident flow. 

The goal of the computation is to approximate the lift and drag coefficients of the flow:
\begin{align}
\label{eq:lift}
    C_L(y) &= \frac{1}{K^\infty(y)} \int_{S} p(y)n(y) \cdot \hat{y} ds, \\
    \label{eq:drag}
    C_D(y) &= \frac{1}{K^\infty(y)} \int_{S} p(y)n(y) \cdot \hat{x} ds,
\end{align}
where $K^\infty(y) = \rho^\infty(y)\|\mathbf{u}^\infty(y)\|^2/2$ is the free-stream kinetic energy with $\hat{y} = [-\sin(\alpha),\cos(\alpha)]$ and $\hat{x} = [\cos(\alpha),\sin(\alpha)]$. 

We remark that this configuration is a prototypical configuration for aerodynamic shape optimization. Here, the underlying maps $\map$ \eqref{eq:map} are the lift coefficient $C_L$ and the drag coefficient $C_D$, realized as functions of a \emph{design space} $Y = [0,1]^{2d}$. A variant of this problem was considered in recent papers \cite{LMR1,LMM1}, where the underlying problem was that of uncertainty quantification (UQ) whereas we consider a shape optimization problem here.

\begin{figure}[htbp]
    \begin{subfigure}{.36\textwidth}
        \centering
        \includegraphics[width=1\linewidth]{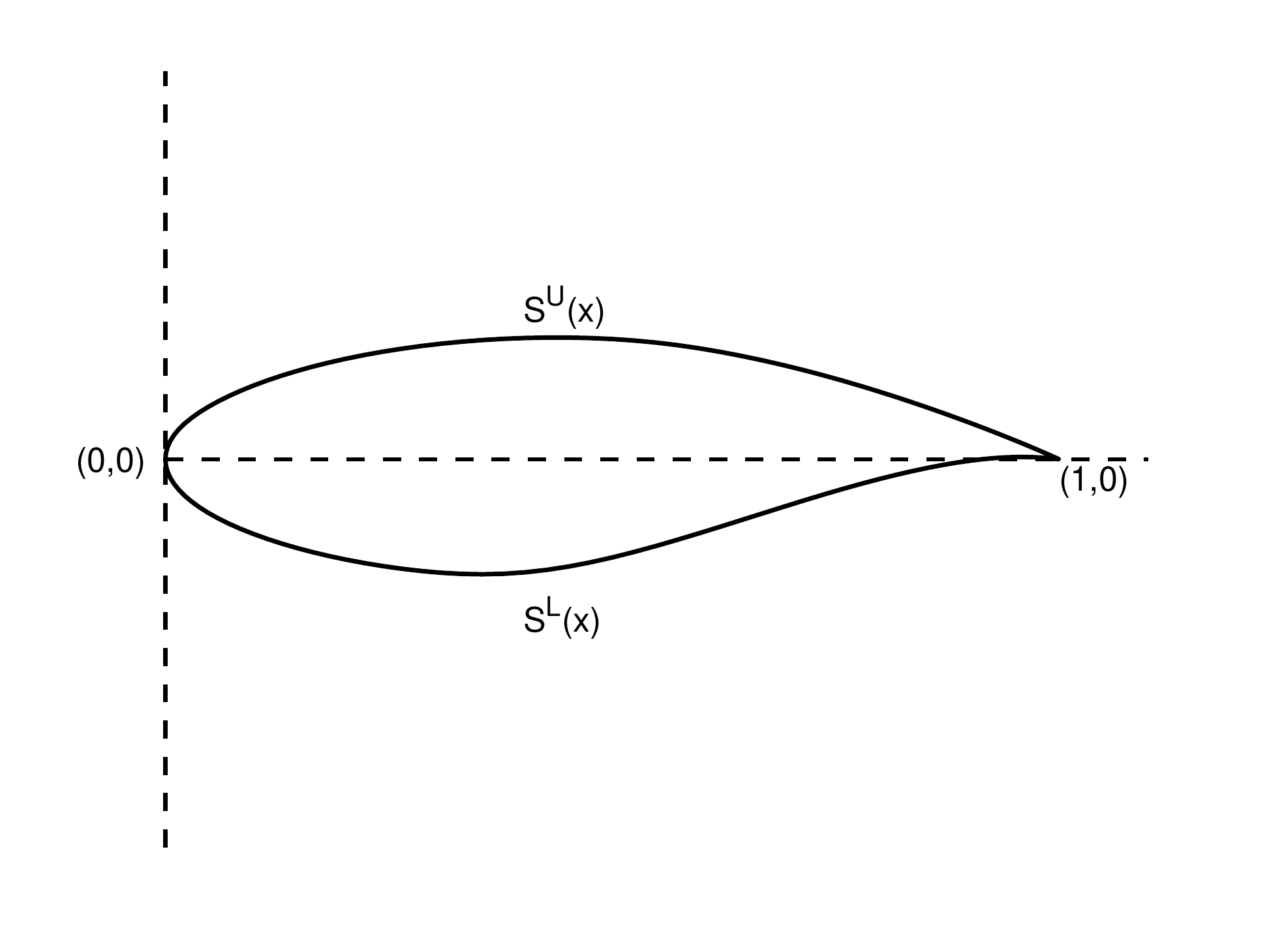}
        \caption{Reference shape}
    \end{subfigure}
    \begin{subfigure}{.3\textwidth}
        \centering\
        \includegraphics[width=1\linewidth]{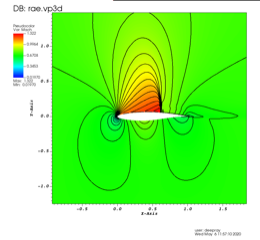}
        \caption{Mach number (Sample)}
    \end{subfigure}
    \begin{subfigure}{.3\textwidth}
        \centering\
        \includegraphics[width=1\linewidth]{{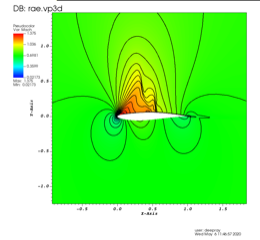}}
        \caption{Mach number (Sample)}
    \end{subfigure}
    
    \caption{Flow past a RAE2822 airfoil with different airfoil shapes Left: Reference airfoil Center and Right: Flow visualized with Mach number for two different samples. }
    \label{fig:afoil}
\end{figure}

The training data is generated by solving the compressible Euler equations in the above domain, with the heavily used NEWTUN code. Two realizations of the flow (with two different airfoil shapes) are depicted in figure \ref{fig:afoil} (center and right). 

We study the $DL_{sob}$ and $DL_{rand}$ algorithms for different sizes of the training set. In figure \ref{fig:airflow_results}, we show the generalization errors for the drag coefficient with $DL_{sob}$ algorithms for different numbers of shape variables (different dimensions $2d$) with $2d=6$ and $2d=10$. As seen from this figure, the generalization error for the $6$-dimensional problem decays at a rate of $0.6$. This is significantly smaller than the rate of $1$, predicted from \eqref{eq:gg2}. However, it is unclear if the Hardy-Krause variation $V_{HK}$ of the drag and the lift will be finite as the flow is very complicated and can even have shocks in this transsonic regime (see figure \ref{fig:afoil}). So, it is very possible that the Hardy-Krause variation blows up in this case (see the discussion on this issue in \cite{LMR1}) and this will certainly reduce the rate of convergence in \eqref{eq:gg2}. Nevertheless, the $DL_{sob}$ algorithm clearly outperforms the $DL_{rand}$ algorithm for this problem. As expected the $DL_{rand}$ converges at the rate of $0.5$. Hence, at the considered resolutions, the $DL_{sob}$ algorithm provides almost an order of magnitude speedup. 

In figure \ref{fig:airflow_results}, we also show the generalization errors for the drag coefficient with $DL_{sob}$ algorithm when $2d=10$, i.e for the $10$-dimensional problem. In this case, the generalization error decays at a slightly worse rate of $\approx 0.56$. This is not unexpected, given the logarithmic corrections in \eqref{eq:gg2}. Very similar results were seen for the lift coefficient and we omit them here. 

Finally, we push the $DL_{sob}$ algorithm to the limits by setting $2d=20$, i.e. considering the $20$-dimensional design space. From \eqref{eq:gg2}, we can see that even if the Hardy-Krause variation of the drag and lift coefficents was finite, the logarithmic terms can dominate for such high dimensions till the unreasonable number of $2^{20}$ training points is reached. Will the $DL_{sob}$ algorithm still outperform the $DL_{rand}$ algorithm in this extreme situation? The answer is provided in figure \ref{fig:airflow_results_20d}, where we plot the generalization errors for the lift coefficient, with respect to number of training samples, for both the $DL_{rand}$ and $DL_{sob}$ algorithms. We observe from this figure that although both deep learning algorithms approximate this underlying map with errors, decaying at the rate of $0.5$, the amplitude of the error is significantly smaller for the $DL_{sob}$ algorithm than for the $DL_{rand}$ algorithm, leading to approximate speedup of a factor of $5$ with the former over the latter. Thus, even in this realistic problem with maps of very low regularity and with a very high underlying dimension, the deep learning algorithm based on low-discrepancy sequences as training points, performs much better than the standard deep learning algorithm, based on random training data. 
\begin{figure}[h!]
\begin{minipage}[t]{0.49\textwidth}	
\includegraphics[width=1.\textwidth]{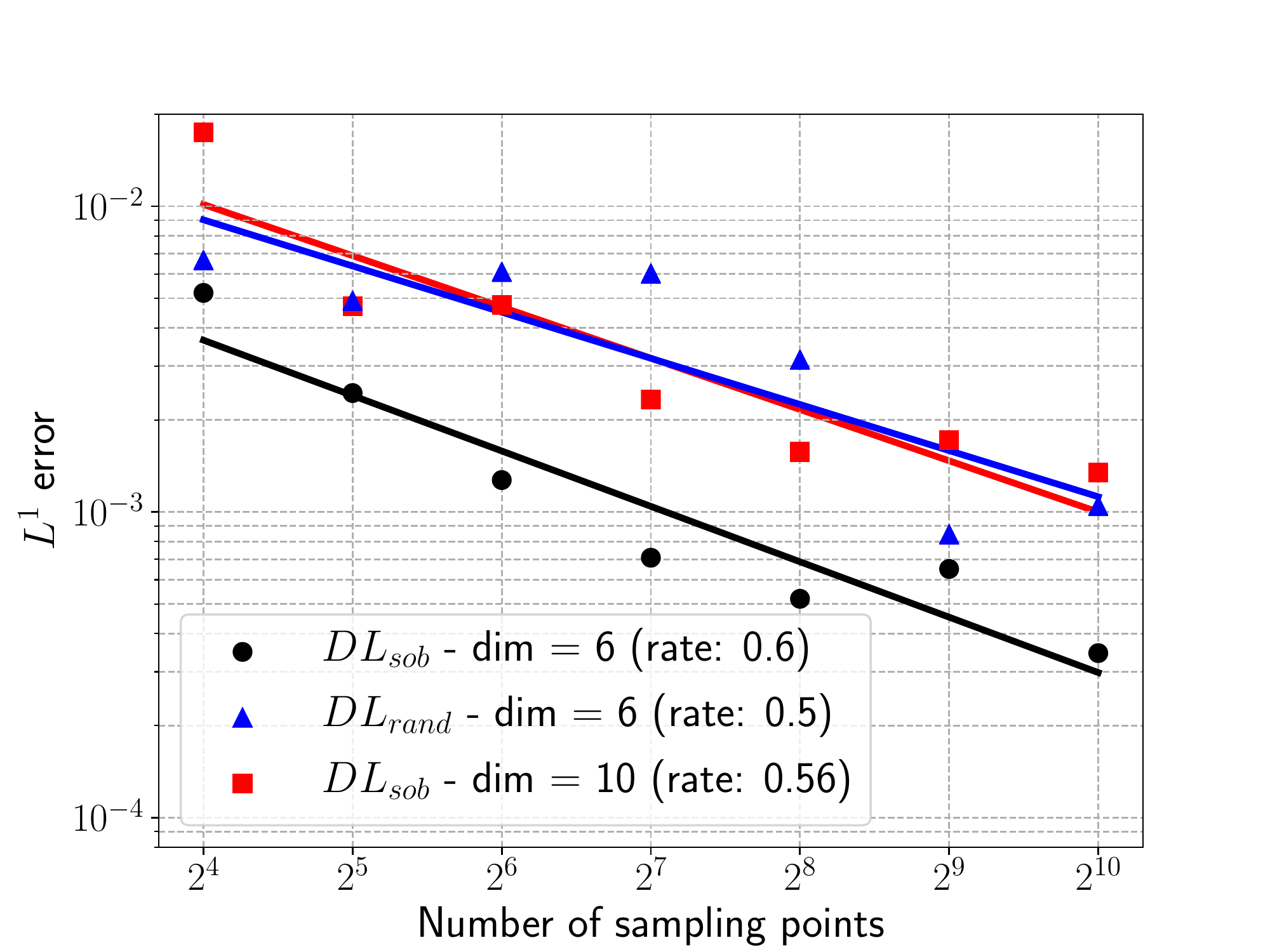}
\caption{$L^1$ generalization errors for approximating the drag coefficient \eqref{eq:drag} using $DL_{sob}$ for input dimensions of $6$ and $10$ compared to using $DL_{rand}$ for a fixed input dimension of $6$.}
\label{fig:airflow_results}
\end{minipage}%
\hspace{0.01\textwidth}
\begin{minipage}[t]{0.49\textwidth}	
\includegraphics[width=1.\textwidth]{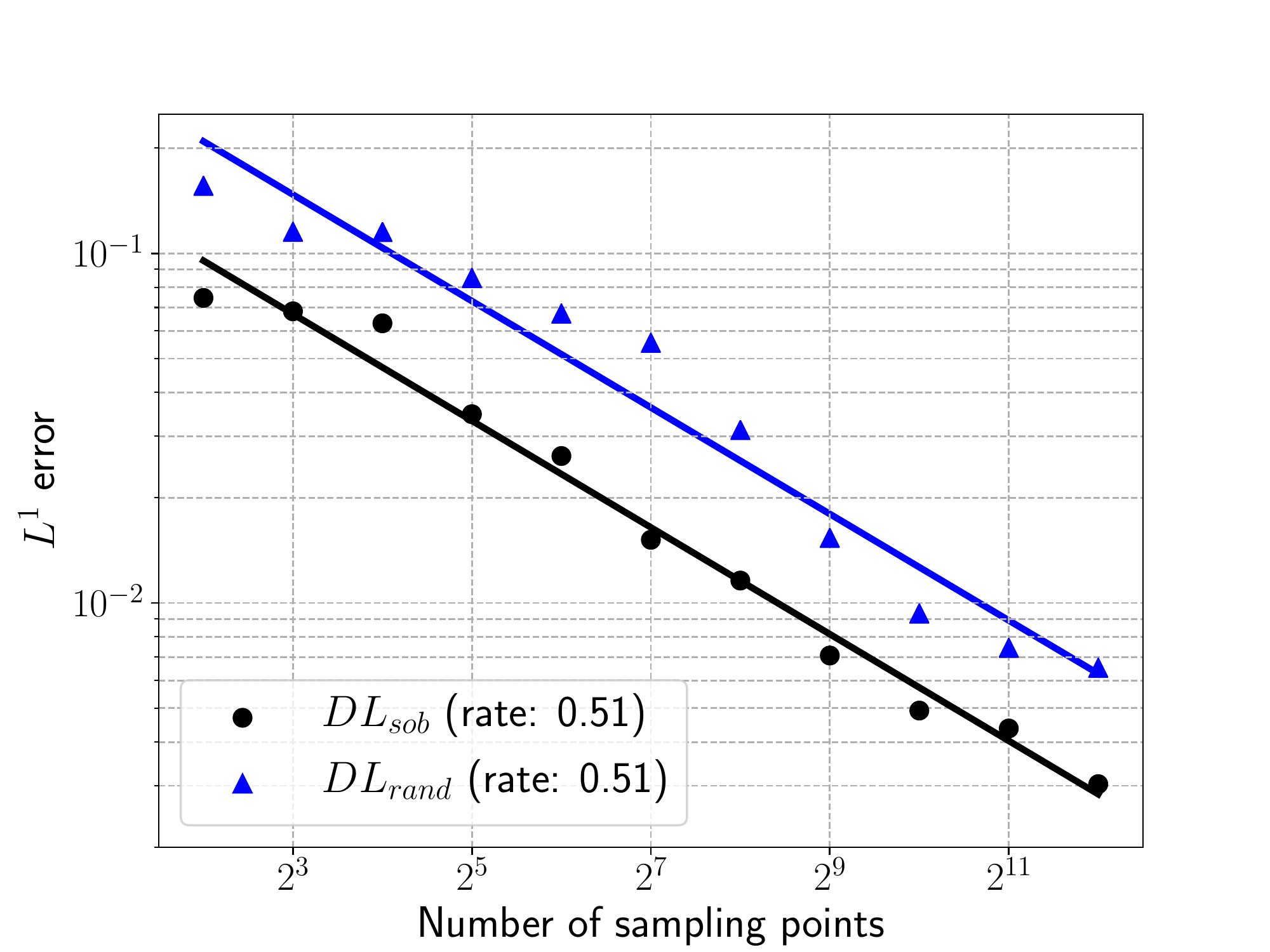}
\caption{$L^1$ generalization errors for approximating the lift coefficient \eqref{eq:lift} for a fixed input dimensions of $20$ using $DL_{sob}$ compared to using $DL_{rand}$.}
\label{fig:airflow_results_20d}
\end{minipage}
\end{figure}
\section{Conclusion}
\label{sec:5}
A key goal in scientific computing is that of simulation or prediction, i.e. evaluation of maps $\map$ \eqref{eq:map} for different inputs. Often, evaluation of these maps require expensive numerical solutions of underlying ordinary or partial differential equations. In the last couple of years, deep learning algorithms have come to the fore for designing surrogates for these computationally expensive maps. However, the standard paradigm of supervised learning, based on randomly chosen training data, suffers from a serious bottleneck for these problems. This is indicated by the bound \eqref{eq:genE2} on the resulting generalization error, i.e. large number of training samples are needed to obtain accurate surrogates. This makes the training process computationally expensive, as each evaluation of the map on a training sample might entail an expensive PDE solve. 

In this article, we follow the recent paper \cite{LMR1} and propose an alternative deep learning algorithm \ref{alg:DL} by considering low-discrepancy sequences (see definition \ref{def:lds}) as the \emph{training set}. We rigorously analyze this algorithm and obtain the upper bounds \eqref{eq:gg1}, \eqref{eq:gg2} on the underlying generalization error \eqref{eq:egen} in terms of the training error \eqref{eq:etrain} and a generalization gap, that for sufficiently regular functions (ones with bounded Hardy-Krause variation) and with sufficiently smooth activation functions for the neural network \eqref{eq:ann1} is shown to decay at a linear rate (with dimension-dependent logarithmic corrections) with respect to the number of training samples. When compared to the standard deep learning algorithm with randomly chosen points, these bounds suggest that at least for maps with desired regularity, up to moderately high underlying dimensions, the proposed deep learning algorithm will significantly outperform the standard version.

\par We test these theoretically derived assertions in a set of numerical experiments that considered approximating maps arising in mechanics (ODEs), computational finance (linear PDEs) and computational fluid dynamics (nonlinear PDEs). From the experiments, we observe that the proposed deep learning algorithm \ref{alg:DL}, with Sobol training points, is indeed accurate and significantly outperforms the one with random training points. The expected linear decay of the error is observed for problems with moderate dimensions. Even for problems with moderately high (up to $20$) dimensions, the proposed algorithm led to much smaller amplitude of errors, than the one based on randomly chosen points. This gain is also seen for maps with low regularity, for which the upper bound \eqref{eq:gg2}
 does not necessarily hold. Thus, the proposed deep learning algorithm is a very simple yet efficient alternative for building surrogates in the context of scientific computing. 
 
The proposed deep learning algorithm can be readily extended to a \emph{multi-level} version to further reduce the generalization error. This was already proposed in a recent paper \cite{LMM1}, although no rigorous theoretical analysis was carried out. This multi-level version would be very efficient for approximating maps with low regularity. We also aim to employ the proposed deep learning algorithms in other contexts in computational PDEs and finance in future papers. 

A key limitation of the theoretical analysis of the proposed algorithm \ref{alg:DL}, as expressed by the bound \eqref{eq:gg2} on the generalization error, is in the very nature of this bound. The bound is on the so-called generalization gap and does not aim to estimate the training error \eqref{eq:etrain} in any way. This is standard practice in theoretical machine learning \cite{AR1}, as the training error entails estimating the result of a stochastic gradient descent algorithm for a highly non-convex, very high-dimensional optimization problem and rigorous guarantees on this method are not yet available. Moreover, in practical terms, this is not an issue. Any reasonable machine learning procedure involves a systematic search of the hyperparameter space and finding suitable hyperparameter configurations that lead to the smallest training error. In particular, this training error is computed a posteriori. Thus, our bounds \eqref{eq:gg2} imply that \emph{as long as we train well, we generalize well}.

Another limitation of the proposed algorithm \ref{alg:DL} is the logarithmic dependence on dimension in the bound \eqref{eq:gg2} on the generalization gap. This dependence is also seen in practice (as shown in figures \ref{fig:multi_dim_BS} and \ref{fig:airflow_results}) and will impede the efficiency of the proposed algorithm for approximating maps in very high dimensions (say $d > 20$). However, this issue has been investigated in some depth in recent literature on Quasi-Monte Carlo methods \cite{schwab_qmc} and references therein and a class of \emph{high-order quasi Monte Carlo} methods have been proposed which can handle this \emph{curse of dimensionality} at least for maps with sufficient regularity. Such methods could alleviate this issue in the current context and we explore them in a forthcoming paper. 

\section*{Acknowledgements.} The research of SM and TKR was partially supported by European Research Council Consolidator grant ERCCoG 770880: COMANFLO. The authors thank Dr. Deep Ray (Rice University, Houston, USA) for helping set up the CFD numerical experiment.

\bibliographystyle{abbrvnat}
\bibliography{refs}

\end{document}